%% file: main.tex
\documentclass{article}

\usepackage[utf8]{inputenc}
\usepackage[T1]{fontenc}
\usepackage{hyperref}
\usepackage{url}
\usepackage{booktabs}
\usepackage{amsfonts}
\usepackage{nicefrac}
\usepackage{microtype}
\usepackage{comment}
\usepackage{graphicx}
\usepackage{nicefrac}
\usepackage{multirow}
\usepackage{mathtools}
\usepackage{relsize}
\usepackage{amsthm}
\usepackage{amssymb}
\usepackage{amsmath}
\usepackage{todonotes}
\usepackage{times}
\usepackage{algorithm}
\usepackage{algorithmic}
\usepackage{float}
\usepackage{dsfont}
\usepackage{numprint}
\usepackage{subfigure}
\usepackage{makecell}
\usepackage{xcolor}
\usepackage{placeins}
\usepackage{setspace}
\usepackage{changepage}
\usepackage[toc,page,header]{appendix}
\usepackage{minitoc}

\newcommand{\norm}[1]{\left\lVert#1\right\rVert}

\usepackage{paralist}

\newtheorem{definition}{Definition}
\newtheorem{lemma}{Lemma}
\newtheorem{theorem}{Theorem}
\newtheorem*{theorem*}{Theorem}
\newtheorem{proposition}{Proposition}

\DeclareMathOperator{\sign}{sign}
\DeclareMathOperator*{\argmax}{argmax}

\DeclareMathOperator*{\Risk}{Risk}
\DeclareMathOperator*{\advRisk}{Risk_{\alpha}}
\DeclareMathOperator*{\PCadvRisk}{PC-Risk_{\alpha}}

\DeclareMathOperator*{\B}{B(\alpha)}
\DeclareMathOperator*{\probmap}{M}
\DeclareMathOperator*{\EoT}{EoT}

\usepackage{chngcntr,etoolbox}
\newcounter{resetdummycounter}
\usepackage[normalem]{ulem}
    \usepackage[preprint, nonatbib]{neurips_2019}

\begin{document}

\title{Theoretical evidence for adversarial robustness through randomization}

% The \author macro works with any number of authors. There are two commands
% used to separate the names and addresses of multiple authors: \And and \AND.
%
% Using \And between authors leaves it to LaTeX to determine where to break the
% lines. Using \AND forces a line break at that point. So, if LaTeX puts 3 of 4
% authors names on the first line, and the last on the second line, try using
% \AND instead of \And before the third author name.

\author{%
  Rafael Pinot$^{1,2}$\quad
  Laurent Meunier$^{1,3}$\quad
  Alexandre Araujo$^{1,4}$\quad \\
  \textbf{Hisashi Kashima}$^{5,6}$\quad
  \textbf{Florian Yger}$^{1}$\quad
  \textbf{C\'edric Gouy-Pailler}$^{2}$\quad
  \textbf{Jamal Atif}$^{1}$\quad \\ \\
  $^1$Université Paris-Dauphine,\ PSL Research University,\ CNRS,\ LAMSADE,\ Paris,\ France \\
  $^2$Institut LIST,\ CEA,\ Universit\'e Paris-Saclay $\quad$ 
  $^3$Facebook AI Research, Paris, France \\
  $^4$Wavestone, Paris, France $\quad$ 
  $^5$Kyoto University,\ Kyoto,\ Japan $\quad$ 
  $^6$RIKEN Center for AIP,\ Japan
}

    \doparttoc % Tell to minitoc to generate a toc for the parts
    \faketableofcontents % Run a fake tableofcontents command for the partocs
    
    \part{} % Start the document part
    \vspace{-1cm}
    % \parttoc % Insert the document TOC

\input{paper.tex}
    
    \setcounter{section}{0}
    \setcounter{proposition}{0}
    \setcounter{lemma}{0}
    \setcounter{theorem}{0}
    \setcounter{definition}{0}
    \setcounter{corollary}{0}
    \setcounter{property}{0}
    \setcounter{example}{0}

\input{matsup.tex}

\end{document}

%% file: paper.tex
% \begin{document}

\maketitle

\vspace{-0.5cm}
\begin{abstract}
    This paper investigates the theory of robustness against adversarial attacks. It focuses on the family of randomization techniques that consist in injecting noise in the network at inference time. These techniques have proven effective in many contexts, but lack theoretical arguments. We close this gap by presenting a theoretical analysis of these approaches, hence explaining why they perform well in practice. More precisely, we make two  new contributions. The first one relates the randomization rate to robustness to adversarial attacks. This result applies for the general family of exponential distributions, and thus extends and unifies the previous approaches. The second contribution consists in devising a new upper bound on the adversarial generalization gap of randomized neural networks. We support our theoretical claims with a set of experiments.
    %  This paper investigates the theory of robustness against adversarial attacks. It focuses on the family of randomization techniques that consist in injecting noise in the network at inference time. These techniques have proven effective in many contexts, but lack theoretical arguments. We close this gap by presenting a theoretical analysis of these approaches, hence explaining why they perform well in practice. More precisely, we make two  new contributions. The first one relates the randomization rate to robustness to adversarial attacks. This result applies for the general family of exponential distributions, and thus extends and unifies the previous approaches. The second contribution consists in devising a new upper bound on the adversarial generalization gap of randomized neural networks. We support our theoretical claims with extensive experiments.
\end{abstract}

\section{Introduction}
Adversarial attacks are some of the most puzzling and burning issues in modern machine learning. An adversarial attack refers to a small, imperceptible change of an input maliciously designed to fool the result of a machine learning algorithm. Since the seminal work of~\cite{Szegedy2013IntriguingPO} exhibiting this intriguing phenomenon in the context of deep learning, a wealth of results have been published on designing attacks~\cite{goodfellow2014explaining,Papernot2016TheLO,moosavi2016deepfool,kurakin2016adversarial,carlini2017towards,moosavi2017universal} and defenses~\cite{goodfellow2014explaining,papernot2016distillation,guo2017countering,meng2017magnet,Samangouei2018DefenseGAN,madry2017towards}), or on trying to understand the very nature of this phenomenon~\cite{fawzi2018empirical,simon2018adversarial,NIPS2018Fawzi,Moosavi2016Robustnessofaclassifier}. Most methods remain unsuccessful to defend against powerful adversaries~\cite{carlini2017towards,madry2018towards,athalye2018obfuscated}. Among the defense strategies, randomization has proven effective in some contexts. It consists in injecting random noise (both during training and inference phases) inside the network architecture, {\em i.e.} at a given layer of the network. Noise can be drawn either from Gaussian~\cite{Xuang2018,lecuyer2018certified,rakin2018parametricnoiseinjection}, Laplace~\cite{lecuyer2018certified}, Uniform~\cite{Xie2017MitigatingAE}, or Multinomial~\cite{pruningDefenseICLR2018} distributions. Remarkably, most of the considered distributions belong to the Exponential family. Albeit these significant efforts, several theoretical questions remain unanswered. Among these, we tackle the following, for which we provide principled and theoretically-founded answers:
%We start by revisiting the notion of robustness against adversarial examples attacks for randomized algorithms, and then draw connections between the Exponential family distributions, and provable robustness to adversarial attacks. In particular we aim to answer the following questions: 
\begin{itemize}
    \vspace{-0.1cm}
    \item[\textbf{Q1:}] To what extent does a noise drawn from the Exponential family preserve robustness\\ (in a sense to be defined) to adversarial attacks?
    \vspace{-0.1cm}
\end{itemize}
\noindent\textbf{A1:} We introduce a definition of robustness to adversarial attacks that is suitable to the randomization defense mechanism. As this mechanism can be  described as a non-deterministic querying process, called probabilistic mapping in the sequel, we propose a formal definition of robustness relying on a metric/divergence between probability measures. A key question arises then about the appropriate metric/divergence for our context. This requires tools for comparing divergences w.r.t. the introduced robustness definition. Renyi divergence turned out to be a measure of choice, since it satisfies most of the desired properties  (coherence, strength, and computational tractability). Finally, thanks to the existing links between the Renyi divergence and the Exponential family, we were able to prove  that methods based on noise injection from the Exponential family  ensures robustness to adversarial examples (cf Theorem~\ref{thm:netrob}) . 
\begin{itemize}
\vspace{-0.1cm}
    \item[\textbf{Q2:}] Can we guarantee a good accuracy under attack for classifiers defended with this\\ kind of noise? 
\vspace{-0.1cm}
\end{itemize}
\noindent\textbf{A2:} We present an upper bound on  the drop of accuracy (under attack) of the methods defended with noise drawn from the Exponential family (cf. Theorem~\ref{thm:bound}). Then, we illustrate this result by training different randomized models with Laplace and Gaussian distributions on CIFAR10/CIFAR100. These experiments highlight the trade-off between accuracy and robustness that depends on the amount of noise one injects in the network. Our theoretical and experimental conclusion is that randomized defenses are competitive (with the current state-of-the-art~\cite{madry2018towards}) given the intensity of noise injected in the network. 

%Our question to the above questions show that randomized methods enjoy principled advantages over other defenses: they can be theoretically studied, they make no assumption about the architecture of the network, they are simple to implement and to understand.

\noindent \textbf{Outline of the paper:} We present in Section~\ref{section::relatedwork} the related work on randomized defenses to adversarial examples. Section~\ref{section::definition} introduces the definition of robustness relying on a metric/divergence between probability measures, and discusses the key role of the Renyi divergence. We state in Section~\ref{sec:main_result} our main results on the robustness and accuracy of Exponential family-based defenses. Section~\ref{section::experiment} presents extensive experiments supporting our theoretical findings. Section~\ref{section::conclusion} provides concluding remarks.
\section{Related works}
\label{section::relatedwork}

%Adversarial examples attacks and associated defenses constitute an engaging and active new field of research
%Even though several explicit or implicit definitions can be found in the literature, {\em e.g.}~\cite{Szegedy2013IntriguingPO,NIPS2018Fawzi,Bubeck2018AdversarialEF}, there is no broadly accepted definition of robustness to adversarial examples attacks. Recently~\cite{NIPS2018Mahloujifar} proposed general definitions and a taxonomy of these. The authors divide the definitions from the literature into three categories:  error-region, prediction-change and corrupted instance. In this paper we introduce a definition of robustness that generalizes the one of prediction-change, 
%in the sense that it relies on probabilistic mappings in arbitrary metric spaces, and is not restricted to classification tasks, as discussed in the sequel.
%in the sense that we consider probabilistic mappings, arbitrary metric spaces, and do not limit the task to the one of classification, as discussed in the sequel. 

% \Jam{Put and develop the following paragraph as a discussion after the def of robustness}
% \del{Finally, contrarily to the previous work, ours doesn't restrict neither the task (regression, classification, reinforcement learning, etc.) nor the type of distribution the perturbation is drawn from.}

Injecting noise into algorithms to improve their robustness has been used for ages in detection and signal processing tasks~\cite{ZozoA99,ChapR04,MitaK98}. It has also been extensively studied in several machine learning and optimization fields, {\em e.g.}~robust optimization~\cite{ben2009robust} and data augmentation techniques~\cite{Perez2017TheEO}. Recently, noise injection techniques have been adopted by the adversarial defense community, especially for neural networks, with very promising results. Randomization techniques are generally oriented towards one of the following objectives: experimental robustness or provable robustness.

\textbf{Experimental robustness:} The first technique explicitly using randomization at inference time as a defense appeared during the 2017 NIPS defense challenge~\cite{Xie2017MitigatingAE}. This method uniformly samples over geometric transformations of the image to select a substitute image to feed the network. Then~\cite{pruningDefenseICLR2018} proposed to use stochastic activation pruning based on a multinomial distribution for adversarial defense. Several papers~\cite{Xuang2018,rakin2018parametricnoiseinjection} propose to inject Gaussian noise directly on the activation of selected layers both at training and inference time. While these works hypothesize that noise injection makes the network robust to adversarial perturbations, they do not provide any formal justification on the nature of the noise they use or on the loss of accuracy/robustness of the  network. \\
\textbf{Provable robustness:} In~\cite{lecuyer2018certified}, the authors proposed a randomization method by exploiting the link between differential privacy~\cite{dwork2014algorithmic} and adversarial robustness. Their framework, called ``randomized smoothing'' \footnote{Name introduced in~\cite{KolterRandomizedSmoothing} which came later than~\cite{lecuyer2018certified}.}, inherits some theoretical results from the differential privacy community allowing them to evaluate the level of accuracy under attack of their method. Initial results from~\cite{lecuyer2018certified} have been refined in~\cite{SecondOrdercertifiedrobustness}, and~\cite{KolterRandomizedSmoothing}. Our work belongs to this line of research. However, our framework does not treat exactly the same class of defenses. Notably, we provide theoretical arguments supporting the defense strategy based on randomization techniques relying on the exponential family, and derive a new bound on the adversarial generalization gap, which completes the results obtained so far on certified robustness. Furthermore, our focus is on the network randomized by noise injection, ``randomized smoothing'' instead uses this network to create a \emph{new} classifier robust to attacks.

Since the initial discovery of adversarial examples, a wealth of non randomized defense approaches have also been proposed, inspired by various machine learning domains such as adversarial training~\cite{goodfellow2014explaining,madry2017towards}, image reconstruction~\cite{meng2017magnet,Samangouei2018DefenseGAN} or robust learning~\cite{goodfellow2014explaining,madry2017towards}. Even if these methods have their own merits, a thorough evaluation made by~\cite{athalye2018obfuscated} shows that most defenses can be easily broken with known powerful attacks~\cite{madry2017towards,carlini2017towards,chen2018ead}. Adversarial training, which consists in training a model directly on adversarial examples, came out as the best defense in average. Defense based on randomization could be overcome by the Expectation Over Transformation technique proposed by~\cite{athalye2017synthesizing} which consists in taking the expectation over the network to craft the perturbation. In this paper, to ensure that our results are not biased by obfuscated gradients, we follow the principles of~\cite{athalye2018obfuscated,carlini2019evaluating} and evaluate our randomized networks with this technique. We show that randomized defenses are still competitive given the intensity of noise injected in the network.

%the focus of  not fall into the scope of this paper (focusing on defenses based on noise injection) and will not be discussed any further.

%But, up to now, no method are fully satisfactory to defend against adversarial attacks.

\section{General definitions of risk and robustness}
\label{section::definition}
\subsection{Risk, robustness and probabilistic mappings}
Let us consider two spaces $\mathcal{X}$ (with norm $\norm{.}_{\mathcal{X}}$), and $\mathcal{Y}$. We consider the classification task that seeks a hypothesis (classifier) $h: \mathcal{X} \rightarrow \mathcal{Y}$ minimizing the risk of $h$ w.r.t. some ground-truth distribution $\mathcal{D}$ over $\mathcal{X}\times\mathcal{Y}$. The risk of $h$ w.r.t $\mathcal{D}$ is defined as 
\begin{align*}
    \Risk(h):= \mathbb{E}_{(x,y)\sim \mathcal{D}}\left[ \mathds{1} \left( h(x) \neq y \right)\right].
\end{align*}
% \paragraph{Adversarial Risk: }
Given a classifier $h: \mathcal{X} \rightarrow \mathcal{Y}$, and some input $x \in \mathcal{X}$ with true label $y_{true} \in \mathcal{Y}$, to generate an adversarial example, the adversary seeks a $\tau$ such that $h(x+\tau) \neq y_{true}$, with some budget $\alpha$ over the perturbation (\emph{i.e} with $\norm{\tau}_{\mathcal{X}} \leq\alpha$). $\alpha$ represents the maximum amount of perturbation one can add to $x$ without being spotted (the perturbation remains humanly imperceptible). 
The overall goal of the adversary is to find a perturbation crafting strategy that both maximizes the risk of $h$, and keeps the values of $\norm{\tau}_{\mathcal{X}}$ small. To measure this risk "under attack" we define the notion of adversarial $\alpha$-radius risk of $h$ w.r.t. $\mathcal{D}$ as follows
\begin{align*}
\advRisk(h):= \mathbb{E}_{(x,y)\sim \mathcal{D}}\left[ \sup_{\norm{\tau}_{\mathcal{X}} \leq \alpha} \mathds{1}\left(h(x+\tau) \neq y\right) \right]\enspace.
\end{align*}

In practice, the adversary does not have any access to the ground-truth distribution. The literature proposed several surrogate versions of $\advRisk(h)$ (see~\cite{NIPS2018Mahloujifar} for more details) to overcome this issue. We focus our analysis on the one used in {\em e.g}~\cite{Szegedy2013IntriguingPO}, or~\cite{NIPS2018Fawzi} denoted $\alpha$-radius prediction-change risk of $h$ w.r.t. $\mathcal{D}_{\mathcal{X}}$ (marginal of $\mathcal{D}$ for $\mathcal{X}$), and defined as   
\begin{align*}
    \PCadvRisk(h):= \mathbb{P}_{x\sim \mathcal{D}_{\mathcal{X}}}\left[\exists \tau \in \B \text{ s.t. } h(x+\tau)\neq h(x) \right]
\end{align*}
where for any $\alpha \geq 0$, \quad $\B :=\{\tau \in \mathcal{X} \text{ s.t. } \norm{\tau}_{\mathcal{X}} \leq \alpha\}\enspace.$
%\paragraph{Generalization to probabilistic mappings:} 

As we will inject some noise in our classifier in order to defend against adversarial attacks, we need to introduce the notion of ``probabilistic mapping''. Let $\mathcal{Y}$ be the output space, and $\mathcal{F}_{\mathcal{Y}}$ a $\sigma$-$ algebra$ over $\mathcal{Y}$. Let us also denote $\mathcal{P}(\mathcal{Y})$ the set of probability measures over $(\mathcal{Y},\mathcal{F}_{\mathcal{Y}})$.

\begin{definition}[Probabilistic mapping] Let $\mathcal{X}$ be an arbitrary space, and $(\mathcal{Y},\mathcal{F}_{\mathcal{Y}})$ a measurable space. A \emph{probabilistic mapping} from $\mathcal{X}$ to $\mathcal{Y}$ is a mapping $\probmap: \mathcal{X} \to \mathcal{P}(\mathcal{Y})$.
To obtain a numerical output out of this \emph{probabilistic mapping}, one needs to sample $y$ according to $\probmap(x)$. %$y\sim \probmap(x)$.
\end{definition} 

This definition does not depend on the nature of $\mathcal{Y}$ as long as $(\mathcal{Y},\mathcal{F}_{\mathcal{Y}})$ is measurable. In that sense, $\mathcal{Y}$ could be either the label space or any intermediate space corresponding to the output of an arbitrary hidden layer of a neural network. Moreover, any mapping can be considered as a probabilistic mapping, whether it explicitly injects noise (as in~\cite{lecuyer2018certified,rakin2018parametricnoiseinjection,pruningDefenseICLR2018}) or not. In fact, any deterministic mapping can be considered as a probabilistic mapping, since it can be characterized by a Dirac measure. Accordingly, the definition of a probabilistic mapping is fully general and equally treats networks with or without noise injection. There exists no definition of robustness against adversarial attacks that comply with the notion of probabilistic mappings. We settle that by generalizing the notion of prediction-change risk initially introduced in~\cite{NIPS2018Mahloujifar} for deterministic classifiers. Let $\probmap$ be a probabilistic mapping from $\mathcal{X}$ to $\mathcal{Y}$, and $d_{\mathcal{P}(\mathcal{Y})}$ some metric/divergence on $\mathcal{P}(\mathcal{Y})$. We define the $(\alpha,\epsilon)$-radius prediction-change risk of $\probmap$ w.r.t. $\mathcal{D}_{\mathcal{X}}$ and $d_{\mathcal{P}(\mathcal{Y})}$ as 
$$\PCadvRisk(\probmap,\epsilon):=  \mathbb{P}_{x\sim \mathcal{D}_{\mathcal{X}}}\left[ \exists \tau \in B(\alpha) \text{ s.t. } d_{\mathcal{P}(\mathcal{Y})}(\probmap(x+\tau),\probmap(x)) > \epsilon \right] \enspace.$$

These three generalized notions allow us to analyze noise injection defense mechanisms (Theorems~\ref{thm:netrob}, and~\ref{thm:bound}). We can also define adversarial robustness (and later adversarial gap) thanks to these notions. 

\begin{definition}[Adversarial robustness]
\label{def::GeneralizedRobustness}
Let $d_{\mathcal{P}(\mathcal{Y})}$ be a metric/divergence on $\mathcal{P}(\mathcal{Y})$. The probabilistic mapping $\probmap$ is said to be
 $d_{\mathcal{P}(\mathcal{Y})}$-$(\alpha, \epsilon, \gamma)$ robust if
$\PCadvRisk(\probmap,\epsilon) \leq \gamma.$ 
\end{definition}

% \begin{definition}[Adversarial gap]
% \label{def::adversarialgap}
% Let $\probmap$ be a probabilistic mapping from $\mathcal{X}$ to $\mathcal{Y}$. The $\alpha$-radius adversarial gap of $\probmap$ is defined as $\text{Gap}_{\alpha}(\probmap) = |\advRisk(K) - \Risk(K)|$
% \end{definition}

It is difficult in general to show that a classifier is $d_{\mathcal{P}(\mathcal{Y})}$-$(\alpha, \epsilon, \gamma)$ robust. However, we can  derive some bounds for particular divergences that will ensure robustness up to a certain level (Theorem~\ref{thm:netrob}). It is worth noting that our definition of robustness depends
on the considered metric/divergence between probability measures. Lemma~\ref{th::PropimpliesRobustness} gives some insights on the monotony of the robustness according to the parameters, and the probability metric/divergence at hand.

\begin{lemma}
\label{th::PropimpliesRobustness}Let $\probmap$ be a probabilistic mapping, and
let  $d_{1}$ and $d_{2}$ be two metrics on $\mathcal{P}(\mathcal{Y})$.
If there exists a non decreasing function $ \phi: \mathbb{R} \to \mathbb{R}$ such that  $\forall \mu_1,\mu_2 \in \mathcal{P}(\mathcal{Y})$, $d_{1}(\mu_1,\mu_2) \leq \phi(d_{2}(\mu_1,\mu_2)) $, then the following assertion holds: 
$$\probmap \text{ is } d_{2}\text{-}(\alpha, \epsilon, \gamma)\text{-robust} \implies \probmap \text{ is }d_{1}\text{-}(\alpha, \phi(\epsilon), \gamma)\text{-robust}$$
\end{lemma}

As suggested in Definition~\ref{def::GeneralizedRobustness} and Lemma~\ref{th::PropimpliesRobustness}, any given choice of metric/divergence will instantiate a particular notion of adversarial robustness and it should be carefully selected. %The joint goals that should naturally lead to the selection of an appropriate metric/divergence are its coherence with the task at hand, and its strength (we define strength as being able to cover a wide number of other metrics regarding Lemma~\ref{th::PropimpliesRobustness}). 

\subsection{On the choice of the metric/divergence for robustness}
\label{subsec:div}
% \Jam{This section needs to be re-organized and simplified. I argue for telling the whole story in a an 'hors d'oeuvre' paragraph before going deeper in the math. To discuss of course. The narration could be as follows, the divergence/metric that generalizes the Bayes risk is the trivial distance with is untractable, but hopefully there is another measure that is related to the Bayes risk is the total variance, which has good properties, however it is still untractable, an hopefully again we have a generalized divergence that provide a rich landscape of behaviors, this one is Renyi...But at the end what we are looking for is a good divergence with goof properties....}

% \Lau{CHANGED}

The aforementioned formulation naturally raises the question of the choice of the metric used to defend against adversarial attacks. 
%At this point, a natural question to be asked is the choice of the metric/divergence we will choose to defend against adversarial attacks. 
The main notions that govern the selection of an appropriate metric/divergence are  \emph{coherence}, \emph{strength}, and \emph{computational tractability}. A metric/divergence is said to be coherent if it naturally fits the task at hand ({\em e.g.} classification tasks are intrinsically linked to discrete/trivial metrics, conversely to regression tasks). The strength of a metric/divergence refers to its ability to cover (dominate) a wide class of others in the sense of Lemma~\ref{th::PropimpliesRobustness}. 
In the following, we will focus on both the total variation metric and the Renyi divergence, that we consider as respectively the most coherent with the classification task using probabilistic mappings, and the strongest divergence. We first discuss how total variation metric is \emph{coherent} with randomized classifiers but suffers from computational issues. Hopefully, the Renyi divergence provides good guarantees about adversarial robustness, enjoys nice \emph{computational properties}, in particular when considering  Exponential family distributions, and is \emph{strong} enough to dominate a wide range of metrics/divergences including total variation.

% \textbf{Dominated measure:} $\mu$ is said to be \emph{dominated} by $\nu$ (denoted $\mu \ll \nu$) if and only if for all $Y \in \mathcal{F}_{\mathcal{Y}}}$, $\ \nu(Y) = 0 \implies \mu(Y)=0$. If $\mu$ is dominated by $\nu$, there is a measurable function $h : \mathcal{Y} \rightarrow [0,+\infty)$ such that for all $Y \in \mathcal{F}_{\mathcal{Y}}}$, $ \mu(Y)=\int_{Y} h \ d\nu$. $ h $ is called the Radon-Nikodym derivative and is denoted $\frac{d \mu}{d \nu}$.

 Let  $\mu_1$ and $\mu_2$ be two measures in $\mathcal{P}(\mathcal{Y})$, both dominated by a third measure $\nu$. The trivial distance $ d_{T}(\mu_1,\mu_2):= \mathds{1}\left(\mu_1 \neq \mu_2\right)$ is the simplest distance one can define between $\mu_1$ and $\mu_2$. In the deterministic case, it is straightforward to compute (since the numerical output of the algorithm characterizes its associated measure), but this is not the case in general. In fact one might not have access to the true distribution of the mapping, but just to the numerical outputs. Therefore, one needs to consider more sophisticated metrics/divergences, such as the total variation distance $ d_{TV}(\mu_1,\mu_2):= \sup_{Y \in \mathcal{F}_{\mathcal{Y}}} |\mu_1 (Y) - \mu_2(Y)|.$ The total variation distance is one of the most broadly used probability metrics. It admits several very simple interpretations, and is a very useful tool in many mathematical fields such as probability theory, Bayesian statistics, coupling or transportation theory. In transportation theory, it can be rewritten as the solution of the Monge-Kantorovich problem with the cost function $c(y_1,y_2) =\mathds{1}\left(y_1 \neq y_2\right)$:
$ \inf\int_{\mathcal{Y}^{2}}\mathds{1}\left(y_1 \neq y_2\right) d\pi(y_1,y_2)\, ,$
where the infimum is taken over all joint probability measures $\pi$ on $(\mathcal{Y}\times \mathcal{Y}, \mathcal{F}_{\mathcal{Y} } \otimes \mathcal{F}_{\mathcal{Y}})$ with marginals $\mu_1$ and $\mu_2$. According to this interpretation, it seems quite natural to consider the total variation distance as a relaxation of the trivial distance on $[0,1]$ (see~\cite{villani2008optimal} for details). In the deterministic case, the total variation and the trivial distance coincides. In general, the total variation allows a finer analysis of the probabilistic mappings than the trivial distance. But it suffers from a high computational complexity. In the following of the paper we will show how to ensure robustness regarding TV distance.

Finally, denoting by $g_1$ and $g_2$ the respective probability distributions w.r.t. $\nu$,  the Renyi divergence of order $\lambda$~\cite{renyi1961} writes as  $d_{R,\lambda}(\mu_1,\mu_2):=\frac{1}{\lambda -1}\log \int_{\mathcal{Y}} g_2(y)  \left(\frac{g_1(y)}{g_2(y)}\right)^{\lambda} d\nu(y).$
The Renyi divergence is a generalized measure defined on the interval $(1,\infty)$, where it equals the Kullback-Leibler divergence when $\lambda \rightarrow 1$ (that will be denoted $d_{KL}$), and the maximum divergence when $\lambda \rightarrow \infty$. It also has the very special property of being non decreasing w.r.t. $\lambda$. This divergence is very common in machine learning, especially in its Kullback-Leibler form as it is widely used as the loss function (cross entropy) of classification algorithms. It enjoys the desired properties  since it bounds the TV distance, and is tractable.  Furthermore, Proposition~\ref{prop:RobustTV} proves that Renyi-robustness implies TV-robustness, making it a suitable surrogate for the trivial distance. 

\begin{proposition}[Renyi-robustness implies TV-robustness]
\label{prop:RobustTV}
Let $\probmap$ be a probabilistic mapping, then $\forall\lambda\geq1$:
$$\probmap \text{ is }  d_{R,\lambda}\text{-}(\alpha, \epsilon, \gamma)\text{-robust} \implies \probmap \text{ is } d_{TV}\text{-}(\alpha, \epsilon', \gamma)\text{-robust}$$
$$\textnormal{ with } \epsilon' = \min \left(\frac{3}{2}\left(\sqrt{1 + \frac{4\epsilon}{9}} - 1\right)^{1/2}, \frac{\exp(\epsilon +1) -1}{\exp(\epsilon +1) +1}\right) \enspace.$$

%And $\forall\lambda\in(0,1)$:
%$$u \text{ is }  d_{R,\lambda}\text{-}(\alpha, \epsilon, \gamma)\text{-robust} \implies u \text{ is } d_{TV}\text{-}(\alpha,\epsilon' , \gamma)\text{-robust} \textnormal{, with } \epsilon'=\sqrt{\frac{2\epsilon}{\alpha}}.$$
\end{proposition}

A crucial property of Renyi-robustness is the \textit{Data processing inequality}. It is a well-known inequality from information theory which states that \textit{``post-processing cannot increase information''}~\cite{cover2012elements,beaudry2011intuitive}. In our case, if we consider a Renyi-robust probabilistic mapping, composing it with a deterministic mapping maintains Renyi-robustness with the same level.

\begin{proposition}[Data processing inequality]
\label{prop::postprocessing} 
Let us consider a probabilistic mapping $\probmap:\mathcal{X}\rightarrow\mathcal{P}(\mathcal{Y})$. Let us also denote $\rho:\mathcal{Y}\rightarrow\mathcal{Y}'$ a deterministic function.
If $U\sim \probmap(x)$ then the probability measure $M'(x)$ s.t $\rho(U) \sim M'(x)$ defines a probabilistic mapping $M':\mathcal{X}\rightarrow\mathcal{P}(\mathcal{Y}')$.

For any $\lambda>1$ if $\probmap$ is $d_{R,\lambda}$-$(\alpha,\epsilon,\gamma)$ robust then $M'$ is also is $d_{R,\lambda}$-$(\alpha,\epsilon,\gamma)$ robust.
\end{proposition}

Data processing inequality will allow us later to inject some additive noise in any layer of a neural network and to ensure Renyi-robustness.

%\section{Noise injection from an Exponential family}
\section{Defense mechanisms based on  Exponential family noise injection}
\label{sec:main_result}
\subsection{Robustness through Exponential family noise injection}
For now, the question of which class of noise to add is treated \textit{ad hoc}. We choose here to investigate one particular class of noise closely linked to the Renyi divergence, namely Exponential family distributions, and demonstrate their interest.
Let us first recall what the Exponential family is.

\begin{definition}[Exponential family]
Let $\Theta$ be an open convex set of $\mathbb{R}^{n}$, and $\theta \in \Theta$. Let $\nu$ be a measure dominated by $\mu$ (either by the Lebesgue or counting measure), it is said to be part of the \emph{Exponential family} of parameter $\theta$ (denoted $E_{F}(\theta,t,k)$) if it has the following probability density function 
$$p_{F}(z,\theta)=\exp\left\{ \langle t(z),\theta \rangle -u(\theta) +k(z) \right\} $$
where $t(z)$ is a sufficient statistic, $k$ a carrier measure (either for a Lebesgue or a counting measure) and $u(\theta)= \log \int_{z} \exp\left\{ <t(z),\theta> +k(z) \right\} dz $.

\end{definition}

To show the robustness of randomized networks with noise injected from the Exponential family, one needs to define the notion of sensitivity for a given deterministic function:
\begin{definition}[Sensitivity of a function]
For any $\alpha\geq0$ and for any $||.||_A$ and $||.||_B$ two norms, the $\alpha$-sensitivity of $f$ w.r.t. $||.||_A$ and $||.||_B$ is defined as $$\Delta^{A,B}_\alpha(f):=\sup\limits_{ x,y \in \mathcal{X}, ||x-y||_{A} \leq \alpha} ||f(x) - f(y) ||_B \enspace.$$
\end{definition}

Let us consider an  $n$-layer feedforward neural network  $\mathcal{N}(.)=\phi^n\circ...\circ\phi^1(.)$. For any $i\in\left[n\right]$, we define $\mathcal{N}_{|i}(.)=\phi^i\circ...\circ\phi^1(.)$ the neural network truncated at layer $i$. Theorem~\ref{thm:netrob} shows that, injecting noise drawn from an Exponential family distribution ensures robustness to adversarial example attacks in the sense of Definition~\ref{def::GeneralizedRobustness}.

% \begin{theorem}[Exponential family ensures robustness]
% \label{thm:netrob}
% Let us denote $\mathcal{N}_{X}^i(.)=\phi^n\circ...\circ\phi^{i+1}(\mathcal{N}_{|i}(.)+X)$ with $X$ a random variable. Then, $\mathcal{N}_{X}^i(.)$ defines a probabilistic mapping that satisfies::

% \begin{itemize}
%     \item If $X\sim E_{F}(\theta,t,k)$  where $t$ and $k$ have non-decreasing modulus of continuity $\omega_t$ and $\omega_k$. Then for any $\alpha \geq 0$, $\probmap$ defines a probabilistic mapping that is $d_{R,\lambda}$-$(\alpha,\epsilon)$ robust with $\epsilon = ||\theta||_2 \omega^{B,2}_t(\Delta^{A,B}_{\alpha}(\phi)) +\omega_k^{B,1}(\Delta^{A,B}_{\alpha}(\phi)) $ where $||.||_2$ is the norm corresponding to the scalar product in the definition of the exponential family density function and $||.||_1$ is here the absolute value on $\mathbb{R}$. The notion of continuity modulus is defined in supplementary material
    
% %Although the Gaussian distribution does not satisfy the modulus of continuity constraint on $t$, we still have robustness for Gaussian noise injection. Let
% \item If $X$ is a centered Gaussian random variable with a non degenerated matrix parameter $\Sigma$. Then for any $\alpha \geq 0$, $\probmap$ defines a probabilistic mapping that is $d_{R,\lambda}$-$(\alpha,\epsilon)$ robust
% with $ \epsilon = \frac{\lambda \Delta^{A,2}_{\alpha}(\phi)^2 }{2 \sigma_{min}(\Sigma) } $ where $||.||_2$ is the canonical Euclidean norm on $\mathbb{R}^n$.
% \end{itemize}
% \end{theorem}

\begin{theorem}[Exponential family ensures robustness]
\label{thm:netrob}
Let us denote $\mathcal{N}_{X}^i(.)=\phi^n\circ...\circ\phi^{i+1}(\mathcal{N}_{|i}(.)+X)$ with $X$ a random variable. Let us also consider two arbitrary norms $||.||_{A}$ and $||.||_{B}$  respectively on $\mathcal{X}$ and on the output space of $\mathcal{N}_{X}^i$.

\begin{itemize}
    \item If $X\sim E_{F}(\theta,t,k)$  where $t$ and $k$ have non-decreasing modulus of continuity $\omega_t$ and $\omega_k$. Then for any $\alpha \geq 0$, $\mathcal{N}_{X}^i(.)$ defines a probabilistic mapping that is $d_{R,\lambda}$-$(\alpha,\epsilon)$ robust with $\epsilon = ||\theta||_2 \omega^{B,2}_t(\Delta^{A,B}_{\alpha}(\phi)) +\omega_k^{B,1}(\Delta^{A,B}_{\alpha}(\phi)) $ where $||.||_2$ is the norm corresponding to the scalar product in the definition of the exponential family density function and $||.||_1$ is the absolute value on $\mathbb{R}$. The notion of continuity modulus is defined in the supplementary material.
    
%Although the Gaussian distribution does not satisfy the modulus of continuity constraint on $t$, we still have robustness for Gaussian noise injection. Let
\item If $X$ is a centered Gaussian random variable with a non degenerated matrix parameter $\Sigma$. Then for any $\alpha \geq 0$, $\mathcal{N}_{X}^i(.)$ defines a probabilistic mapping that is $d_{R,\lambda}$-$(\alpha,\epsilon)$ robust
with $ \epsilon = \frac{\lambda \Delta^{A,2}_{\alpha}(\phi)^2 }{2 \sigma_{min}(\Sigma) } $ where $||.||_2$ is the canonical Euclidean norm on $\mathbb{R}^n$.
\end{itemize}
\end{theorem}

In simpler words, the previous theorem ensures stability in the neural network when injecting noise w.r.t. the distribution of the output. Intuitively, if two inputs are close w.r.t. $\norm{.}_{A}$, the output distributions of the network will be close in the sense of Renyi divergence. It is well known that in the case of deterministic neural networks, the Lipschitz constant becomes bigger as the number of layers increases~\cite{gouk2018regularisation}. By injecting noise at layer $i$, the notion of robustness only depends on the sensitivity of the first $i$ layers of the network and not the following ones. In that sense, randomization provides a more precise control on the ``continuity'' of the neural network. In the next section, we show that thanks to the notion of robustness w.r.t. probabilistic mappings, one can bound the loss of accuracy of a randomized neural network when it is attacked. 

\subsection{Bound on the generalization gap under attack}

The notions of risk and adversarial risk can easily be generalized to encompass probabilistic mappings. % (e.g by analogy with PAC Bayes theory).

\begin{definition}[Risks for probabilistic mappings]
 Let $\probmap$ be a probabilistic mapping from $\mathcal{X}$ to $\mathcal{Y}$, the risk and the $\alpha$-radius adversarial risk of $\probmap$ w.r.t. $\mathcal{D}$ are defined as 
\begin{align*}
&\Risk(\probmap):= \mathbb{E}_{(x,y)\sim \mathcal{D}}\left[ \mathbb{E}_{y'\sim \probmap(x)} \left[ \mathds{1} \left( y' \neq y \right)\right]\right]\\
&\advRisk(\probmap):= \mathbb{E}_{(x,y)\sim \mathcal{D}}\left[ \sup_{\norm{\tau}_{\mathcal{X}} \leq \alpha}\mathbb{E}_{y'\sim \probmap(x+\tau)} \left[ \mathds{1} \left( y' \neq y \right)\right]\right]\enspace.
\end{align*}

% Regarding the randomization of neural networks, we define the following notions of generalization. For $\Tilde{\mathcal{N}}^i_X$ a randomized neural network with some noise $X$ injected at layer $i$ in:
% \begin{itemize}

%     % \item  Generalization error: $Err(\mathcal{N})=\mathbb{E}_{(x,y)}(1\!\!1_{\mathcal{N}(x)\neq y})$

%     % \item Adversarial generalization error: $Err_\alpha(\mathcal{N})=\mathbb{E}_{(x,y)}(\sup_{\tau/\norm{\tau}\leq\alpha}1\!\!1_{\mathcal{N}(x+\tau)\neq y})$

%     \item Generalization error for the probabilistic mapping $\Tilde{\mathcal{N}}^i_X$: 
    
%     $$\overline{Err}(\Tilde{\mathcal{N}}^i_X)=\mathbb{E}_{(x,y)}(\mathbb{E}_{X}(1\!\!1_{\Tilde{\mathcal{N}}^i_X(x)\neq y}))$$%=\mathbb{E}_{X}(Err(\Tilde{\mathcal{N}}^i_X))$$
    
%     \item Adversarial generalization error for the probabilistic mapping $\Tilde{\mathcal{N}}^i_X$:
    
%     $$\overline{Err}_\alpha(\Tilde{\mathcal{N}}^i_X)=\mathbb{E}_{(x,y)}(\sup_{\tau/\norm{\tau}\leq\alpha}\mathbb{E}_{X}(1\!\!1_{\Tilde{\mathcal{N}}^i_X(x+\tau)\neq y}))$$%=\mathbb{E}_{X}(Err_\alpha(\Tilde{\mathcal{N}}^i_X))$$
% \end{itemize}

\end{definition}

The definition of adversarial risk for a probabilistic mapping can be matched with the concept of Expectation over Transformation (EoT) attacks~\cite{athalye2018obfuscated}. Indeed, EoT attacks aim at computing the best opponent in expectation for a given random transformation. In the adversarial risk definition, the adversary chooses the perturbation which has the greatest probability to fool the model, which is a stronger objective than the EoT objective. Theorem~\ref{thm:bound} provides a bound on the gap between the adversarial risk and the regular risk:

\begin{theorem}[Adversarial generalization gap bound in the randomized setting]

\label{thm:bound}
Let $\probmap$ be the probabilistic mapping at hand. Let us suppose that  $\probmap$ is $d_{R,\lambda}$-$(\alpha,\epsilon)$ robust for some $\lambda\geq1$ then:

$$|\advRisk(\probmap)-\Risk(\probmap)|\leq 1-e^{-\epsilon}\mathbb{E}_x\left[e^{-H(\probmap(x))}\right]$$
where $H$ is the Shannon entropy $H(p)=-\sum_i p_i \log(p_i)\enspace.$
\end{theorem}
This theorem gives a control on the loss of accuracy under attack w.r.t. the robustness parameter $\epsilon$ and the entropy of the predictor. It provides a tradeoff between the quantity of noise added in the network and the accuracy under attack. Intuitively, when the noise increases, for any input, the output distribution tends towards the uniform distribution, then, $\epsilon\rightarrow0$ and $H(\probmap(x))\rightarrow \log(K)$, and the risk and the adversarial risk both tends to $\frac{1}{K}$ where $K$ is the number of classes in the classification problem. On the opposite, if no noise is injected, for any input, the output distribution is a  Dirac distribution, then, if the prediction for the adversarial example is not the same as for the regular one, $\epsilon\rightarrow\infty$ and $H(\probmap(x))\rightarrow 0$. Hence, the noise needs to be designed both to preserve accuracy and robustness to adversarial attacks. In the Section~\ref{section::experiment}, we give an illustration of this bound when $\probmap$ is a neural network with noise injection at input level as presented in Theorem~\ref{thm:netrob}.

% \begin{align*}
% |\overline{Err}_\alpha(\mathcal{\Tilde{N}}_X^i)-\overline{Err}(\mathcal{\Tilde{N}}_X^i)|  &= |\mathbb{E}_{(x,y)}( \sup_{\tau/\norm{\tau}\leq\alpha} \mathbb{E}_X(1\!\!1_{\Tilde{\mathcal{N}}^i_X(x+\tau)\neq y})-  \mathbb{E}_X(1\!\!1_{\Tilde{\mathcal{N}}^i_X(x)\neq y}))| \\
% &=|\mathbb{E}_{(x,y)}( \sup_{\tau/\norm{\tau}\leq\alpha} \mathbb{E}_{X_1,X_2}(1\!\!1_{\Tilde{\mathcal{N}}^i_{X_1}(x+\tau)\neq y}-  1\!\!1_{\Tilde{\mathcal{N}}^i_{X_2}(x)\neq y}))|\\
% &\leq\mathbb{E}_{(x,y)}( \sup_{\tau/\norm{\tau}\leq\alpha} \mathbb{E}_{X_1,X_2}(|1\!\!1_{\Tilde{\mathcal{N}}^i_X(x+\tau)\neq y}-  1\!\!1_{\Tilde{\mathcal{N}}^i_X(x)\neq y}|))\\
% &=\mathbb{E}_{(x,y)}(\sup_{\tau/\norm{\tau}\leq\alpha}\mathbb{P}_{X_1,X_2}(\Tilde{\mathcal{N}}^i_{X_1}(x+\tau)=\Tilde{\mathcal{N}}^i_{X_2}(x)))
% \end{align*}

% where $X_1$ and $X_2$ are two independent samples with the same law than $X$.

% For two discrete random independent variables of law $P=(p_1,...,p_K)$ and $Q=(q_1,...,q_K)$: 
% $$\mathbb{P}(P=Q)=\sum_{i=1}^K p_i q_i \geq \exp{(\sum_{i=1}^K p_i \log q_i)}=\exp{(-d_{KL}(P||Q)-H(P))}$$

\section{Numerical experiments}
\label{section::experiment}
To illustrate our theoretical findings, we train randomized neural networks with a simple method which consists in injecting a noise drawn from an Exponential family distribution in the image during training and inference. This section aims to answer \textbf{Q2} stated in the introduction, by tackling the following sub-questions:
\begin{itemize}
    \item[\textbf{Q2.1:}] How does the randomization impact the accuracy of the network? And, how does the theoretical trade-off between accuracy and robustness apply in practice? 
    \item[\textbf{Q2.2:}] What is the accuracy under attack of randomized neural networks against powerful iterative attacks? And how does randomized neural networks compare to state-of-the-art defenses given the intensity of the injected noise? 
\end{itemize}

\subsection{Experimental setup}
We present our results and analysis on  CIFAR-10, CIFAR-100 \cite{krizhevsky2009learning} and ImageNet datasets \cite{imagenet_cvpr09}. For CIFAR-10 and CIFAR-100 \cite{krizhevsky2009learning}, we used a Wide ResNet architecture \cite{ZagoruykoK16} which is a variant of the ResNet model from \cite{he2016deep}. We use 28 layers with a widen factor of 10. We train all networks for 200 epochs, a batch size of 400, dropout 0.3 and Leaky Relu activation with a slope on $\mathbb{R}^-$ of 0.1. We minimize the Cross Entropy Loss with Momentum 0.9 and use a piecewise constant learning rate of 0.1, 0.02, 0.004 and 0.00008 after respectively 7500, 15000 and 20000 steps. The networks achieve for CIFAR10 and 100 a TOP-1 accuracy of 95.8\% and 79.1\% respectively on test images. For ImageNet \cite{imagenet_cvpr09}, we use an Inception ResNet v2 \cite{szegedy2017inception} which is the sate of the art architecture for this dataset and achieve a TOP-1 accuracy of 80\%. For the training of ImageNet, we use the same hyper parameters setting as the original implementation. We train the network for 120 epochs with a batch size of 256, dropout 0.8 and Relu as activation function. All evaluations were done with a single crop on the non-blacklisted subset of the validation set.

To transform these classical networks to probabilistic mappings, we inject noise drawn from Laplace and Gaussian distributions, each with various standard deviations. While the noise could theoretically be injected anywhere in the network, we inject the noise on the image for simplicity. More experiments with noise injected in the first layer of the network are presented in the supplementary material. To evaluate our models under attack, we use three powerful iterative attacks with different norms: \emph{ElasticNet} attack (EAD)~\cite{chen2018ead} with $\ell_1$ distortion, \emph{Carlini\&Wagner} attack (C\&W)~\cite{carlini2017towards} with $\ell_2$ distortion and \emph{Projected Gradient Descent} attack (PGD)~\cite{madry2017towards} with $\ell_\infty$ distortion. All standard deviations and attack intensities are in between $-1$ and $1$. Precise descriptions of our numerical experiments and of the attacks used for evaluation are deferred to the supplementary material.

\paragraph{Attacks against randomized defenses:} It has been pointed out by \cite{athalye2017synthesizing,carlini2019evaluating} that in a white box setting, an attacker with a complete knowledge of the system will know the distribution of the noise injected in the network. As such, to create a stronger adversarial example, the attacker can take the expectation of the loss or the logits of the randomized network during the computation of the attack. This technique is called Expectation Over Transformation ($\EoT$) and we use a Monte Carlo method with $80$ simulations to approximate the best perturbation for a randomized network. 

\subsection{Experimental results}

\paragraph{Trade-off between accuracy and intensity of noise (Q2.1):}

\begin{figure}[t]
\centering
\subfigure[]{
    \includegraphics[width=.31\textwidth]{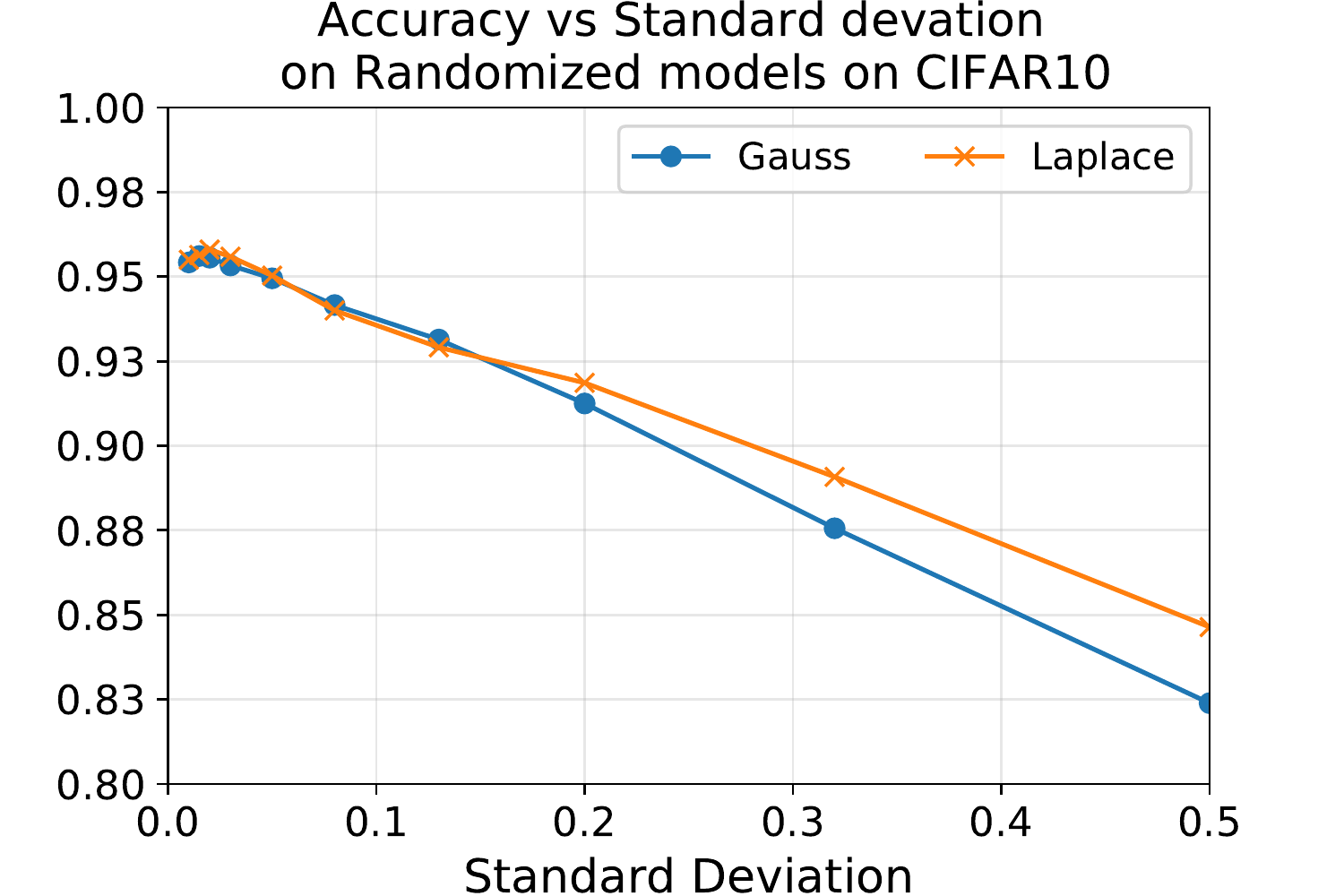}
    \label{fig:acc_sd_CIFAR10}
    }
\subfigure[]{
    \includegraphics[width=.31\textwidth]{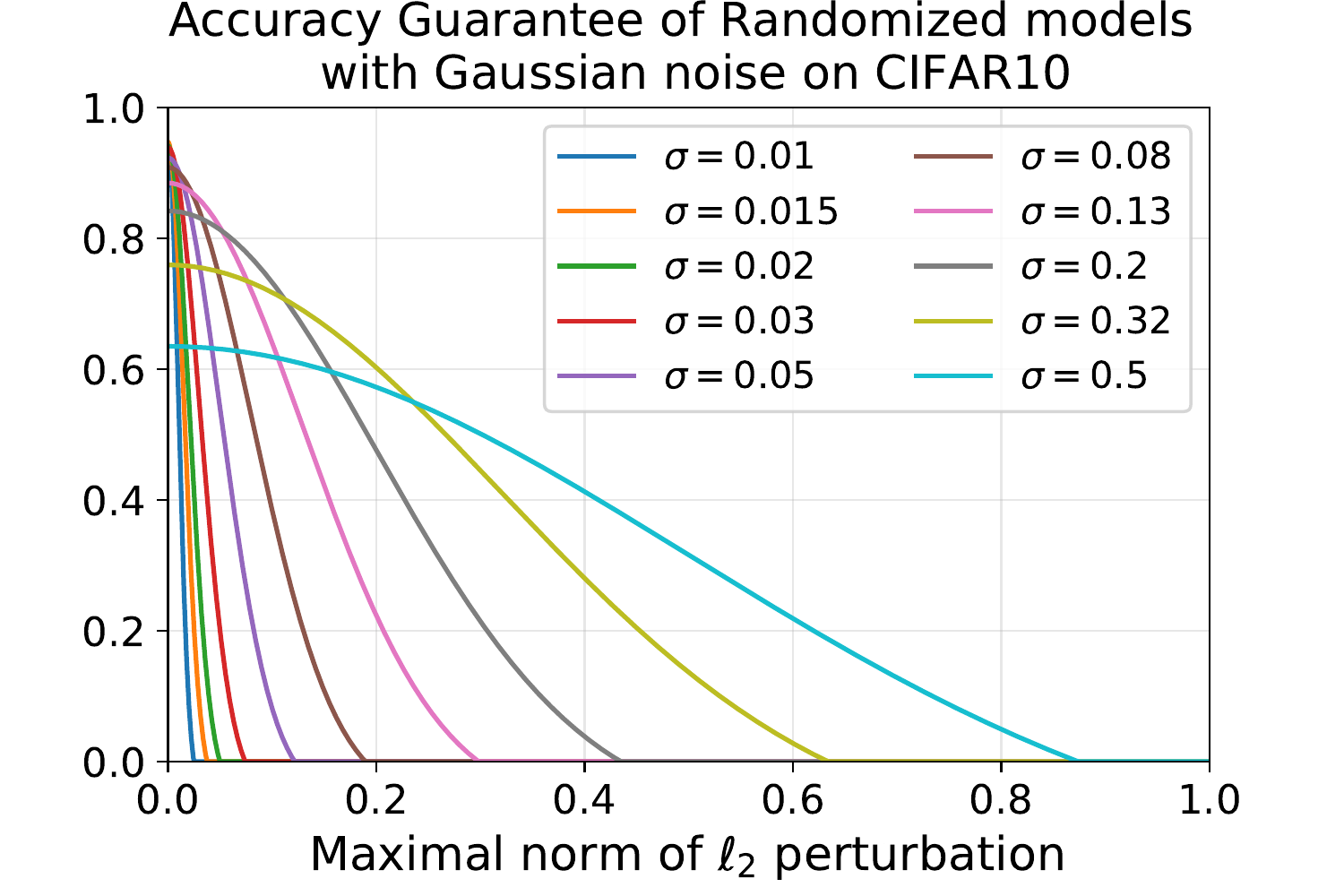}
    \label{fig:gauss_certif_CIFAR10}
    }
\subfigure[]{
    \includegraphics[width=.31\textwidth]{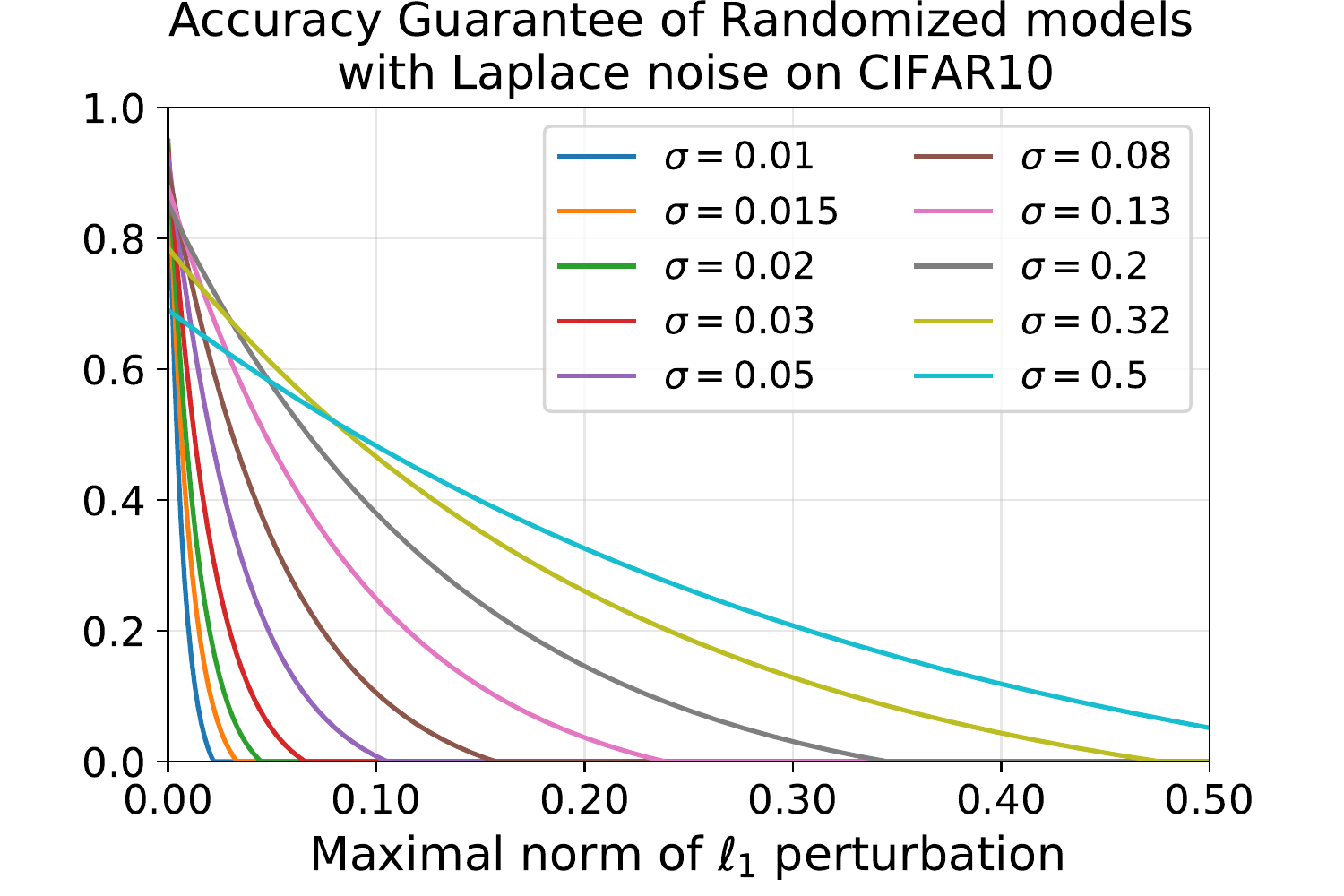}
    \label{fig:laplace_certif_CIFAR10}
    }
 \label{fig:cifar10_results}
\caption{(a) Impact of the standard deviation of the injected noise on accuracy in a randomized model on CIFAR-10 dataset with a Wide ResNet architecture. (b) and (c) illustration of the guaranteed accuracy of different randomized models with Gaussian (b) and Laplace (c) noises given the norm of the adversarial perturbation.}
\end{figure}

When injecting noise as a defense mechanism, regardless of the distribution it is drawn from, we observe (as in Figure~\ref{fig:acc_sd_CIFAR10}) that the accuracy decreases when the noise intensity grows. In that sense, noise needs to be calibrated to preserve both accuracy and robustness against adversarial attacks, i.e. it needs to be large enough to preserve robustness and small enough to preserve accuracy. Figure~\ref{fig:acc_sd_CIFAR10} shows the loss of accuracy on CIFAR10 from $0.95$ to $0.82$ (respectively $0.95$ to $0.84$) with noise drawn from a Gaussian distribution (respectively Laplace) with a standard deviation from $0.01$ to $0.5$. Figure~\ref{fig:gauss_certif_CIFAR10} and \ref{fig:laplace_certif_CIFAR10} illustrate the theoretical lower bound on accuracy under attack of Theorem~\ref{thm:bound} for different distributions and standard deviations. The term in entropy of Theorem~\ref{thm:bound} has been estimated using a Monte Carlo method with $10^4$ simulations. The trade-off between accuracy and robustness from Theorem~\ref{thm:bound} thus appears w.r.t the noise intensity. With small noises, the accuracy is high, but the guaranteed accuracy drops fast w.r.t the magnitude of the adversarial perturbation. Conversely, with bigger noises, the accuracy is lower but decreases slowly w.r.t the magnitude of the adversarial perturbation. These Figures also show that Theorem~\ref{thm:bound} gives strong accuracy guarantees against small adversarial perturbations. Next paragraph shows that in practice, randomized networks achieve much higher accuracy under attack than the theoretical bound, and keep this accuracy against much larger perturbations.

\paragraph{Performance of randomized networks under attacks and comparison to state of the art (Q2.2):}\label{sec:perf_under_attack}

While Figure~\ref{fig:gauss_certif_CIFAR10} and \ref{fig:laplace_certif_CIFAR10} illustrated a theoretical robustness against growing adversarial perturbations, Table~\ref{tab:accuracy_under_attack} illustrates this trade-off experimentally. It compares the accuracy obtained under attack by a deterministic network with the one obtained by randomized networks with Gaussian and Laplace noises both with low ($0.01$) and high ($0.5$) standard deviations. Randomized networks with a small noise lead to no loss in accuracy with a small robustness while high noise leads to a higher robustness at the expense of loss of accuracy ($\sim11$ points). 

\begin{table}[t]
  \centering
  \caption{Accuracy under attack on the CIFAR-10 dataset with a randomized Wide ResNet architecture. We compare the accuracy on natural images and under attack with different noise over 3 iterative attacks (the number of steps is next to the name) made with 80 Monte Carlo simulations to compute EoT attacks. The first line is the baseline, no noise has been injected.}
    \begin{tabular}{lccccc}
    \toprule
    \textbf{Distribution} & \textbf{Sd} & \textbf{Natural} & \textbf{$\ell_1$ -- EAD 60} & \textbf{$\ell_2$ -- C\&W 60} & \textbf{$\ell_\infty$ -- PGD 20} \\
    \midrule
    - & - & 0.958 & 0.035 & 0.034 & 0.384 \\
    \midrule
    \multirow{2}[0]{*}{Normal} & 0.01 & 0.954 & 0.193 & 0.294 & 0.408 \\
          & 0.50 & 0.824 & 0.448 & 0.523 & 0.587 \\
    \midrule
    \multirow{2}[0]{*}{Laplace} & 0.01 & 0.955 & 0.208 & 0.313 & 0.389 \\
          & 0.50 & 0.846 & 0.464 & 0.494 & 0.589 \\
    \bottomrule
    \end{tabular}%
  \label{tab:accuracy_under_attack}%
\end{table}%

% \subsection{Comparison with Adversarial Training (Q4)}\label{sec:adv_madry}

\begin{table}[t]
  \caption{Accuracy under attack of randomized neural network with different distributions and standard deviations versus adversarial training by Madry et al. \cite{madry2017towards}. The PGD attack has been made with 20 step, an epsilon of 0.06 and a step size of 0.006 (input space between $-1$ and $+1$). The Carlini\&Wagner attack uses 30 steps, 9 binary search steps and a 0.01 learning rate. The first line refers to the baseline without attack.}
  \label{Results}
  \centering
  \begin{tabular}{ccccccc}
    \toprule
      & & \multirow{2}[0]{*}{\thead{\textbf{Madry et al.} \\ \cite{madry2017towards}}} & \multirow{2}[0]{*}{\textbf{Normal 0.32}} & \multirow{2}[0]{*}{\textbf{Laplace 0.32}} & \multirow{2}[0]{*}{\textbf{Normal 0.5}} & \multirow{2}[0]{*}{\textbf{Laplace 0.5}} \\
     \textbf{Attack} & \textbf{Steps} & & & \\
    \midrule
    -  & - & 0.873 & 0.876 & 0.891 & 0.824 & 0.846 \\ 
    $\ell_\infty$ -- PGD & 20 & 0.456 & 0.566 & 0.576 & 0.587 & 0.589 \\
    $\ell_2$ -- C\&W & 30 & 0.468 & 0.512 & 0.502 & 0.489 & 0.479 \\
    \bottomrule
  \end{tabular}
  \label{table:madry_vs_random}
\end{table}
Finally, Table~\ref{table:madry_vs_random} compares the accuracy and the accuracy under attack of randomized networks with Gaussian and Laplace distributions for different standard deviations against adversarial training from Madry et al.~\cite{madry2017towards}. We observe that the accuracy on natural images of both noise injection methods are similar to the one from~\cite{madry2017towards}. Moreover, both methods are more robust than adversarial training to PGD and C\&W attacks. As with all the experiments, to construct an EoT attack,  we use 80 Monte Carlo simulations at every step of PGD and C\&W attacks. These experiments show that randomized defenses can be competitive given the intensity of noise injected in the network. Note that these experiments have been led with $\EoT$ of size 80. For much bigger sizes of $\EoT$ these results would be mitigated. Nevertheless, the accuracy would never drop under the bounds illustrated in Figure~\ref{fig:cifar10_results}, since Theorem~\ref{thm:bound} gives a bound that on the worst case attack strategy (including $\EoT$).

\section{Conclusion and future works}
\label{section::conclusion}
This paper brings new contributions to the field of provable defenses to adversarial attacks. Principled answers have been provided to key questions on the interest of randomization techniques, and on their loss of accuracy under attack. The obtained bounds have been illustrated in practice by conducting thorough experiments on baseline datasets such as CIFAR and ImageNet. We show in particular that a simple method based on injecting noise drawn from the Exponential family is competitive compared to baseline approaches while leading to provable guarantees. Future work will focus on investigating other noise distributions belonging or not to the Exponential family, combining randomization with more sophisticated defenses and on devising new tight bounds on the adversarial generalization gap.

\bibliographystyle{abbrv}
\bibliography{biblio}

% \end{document}

%% file: matsup.tex
% \title{Theoretical evidence for adversarial robustness through randomization \\- Supplementary Material -}
% \maketitle
\newpage
\begin{adjustwidth}{-1.5cm}{-1cm}

\appendix
\addcontentsline{toc}{section}{Appendix} % Add the appendix text to the document TOC
\part{Supplementary Material} % Start the appendix part
\noptcrule 
\parttoc

\section{Notations and definitions}

Let us consider an output space $\mathcal{Y}$, and $\mathcal{F}_{\mathcal{Y}}$ a $\sigma$-$ algebra$ over $\mathcal{Y}$. We denote $\mathcal{P}(\mathcal{Y})$ the set of probability measures over $(\mathcal{Y},\mathcal{F}_{\mathcal{Y}})$.
Let $(\mathcal{Y'},\mathcal{F}_{\mathcal{Y'}})$ be a second measurable space, and $\phi$ a measurable mapping from $(\mathcal{Y},\mathcal{F}_{\mathcal{Y}})$ to $(\mathcal{Y'},\mathcal{F}_{\mathcal{Y'}})$. Finally Let us consider $\mu,\nu$ two measures on $(\mathcal{Y},\mathcal{F}_{\mathcal{Y}})$.

\textbf{Dominated measure:} $\mu$ is said to be dominated by $\nu$ (denoted $\mu \ll \nu$) if for all $Y \in \mathcal{F}_{\mathcal{Y}}$, $\ \nu(Y) = 0 \implies \mu(Y)=0$. If $\mu$ is dominated by $\nu$, there is a measurable function $h : \mathcal{Y} \rightarrow [0,+\infty)$ such that for all $Y \in \mathcal{F}_{\mathcal{Y}}$, $ \mu(Y)=\int_{Y} h \ d\nu$. $ h $ is called the Radon-Nikodym derivative of $\mu$ w.r.t. $\nu$ and is denoted $\frac{d \mu}{d \nu}$.
   
\textbf{Push-forward measure:} the push-forward measure of $\nu$ by $\phi$ (denoted $\phi \# \nu$) is the measure on $(\mathcal{Y'},\mathcal{F}_{\mathcal{Y'}})$ such that $\forall Z \in \mathcal{F}_{\mathcal{Y'}}, \ \phi \# \nu(Z) = \nu(\phi^{-1}(Z)) $.

\textbf{Convolution product:} the convolution of $\nu$ with $\mu$, denoted $\nu * \mu$ is the push-forward measure of $\nu \otimes \mu$ by the addition on $\mathcal{Y}$. Since the convolution between functions is defined accordingly, we use $*$ indifferently for measures and simple functions. 

\textbf{Modulus of continuity:} Let us consider $f:(E,\norm{.}_E)\rightarrow(F,\norm{.}_F)$. $f$ admits a non-decreasing modulus of continuity regarding $\norm{.}_E$ and $\norm{.}_F$ if there exists a non-decreasing function $\omega^{E,F}_f:\mathbb{R^+}\rightarrow\mathbb{R^+}$ such as for all $x,y\in E$, $||f(y)-f(x)||_F\leq \omega^{E,F}_f(||x-y||_E)$.

\newpage
\section{Main proofs}
% \todo{Ca vous va comme formulation ?}
As a trade-off between completeness and conciseness, we delayed the proofs of our theorems to this section.
\subsection{Proof of Lemma~\ref{th::PropimpliesRobustness-appendix}}

\begin{lemma}
\label{th::PropimpliesRobustness-appendix}Let $\probmap$ be a probabilistic mapping, and
let  $d_{1}$ and $d_{2}$ be two metrics on $\mathcal{P}(\mathcal{Y})$.
If there exists a non decreasing function $ \phi: \mathbb{R} \to \mathbb{R}$ such that  $\forall \mu_1,\mu_2 \in \mathcal{P}(\mathcal{Y})$, $d_{1}(\mu_1,\mu_2) \leq \phi(d_{2}(\mu_1,\mu_2)) $, then the following holds: 
$$\probmap \text{ is } d_{2}\text{-}(\alpha, \epsilon, \gamma)\text{-robust} \implies \probmap \text{ is }d_{1}\text{-}(\alpha, \phi(\epsilon), \gamma)\text{-robust}$$
\end{lemma}

\begin{proof}
Let us consider a probabilistic mapping $\probmap$, $x \sim \mathcal{D}$, and $\tau \in B(\alpha)$, one has $d_{1}(\probmap(x),\probmap(x +\tau)) \leq \phi(d_{2}(\probmap(x),\probmap(x +\tau))) \leq \phi(\epsilon).$ Hence  $\mathbb{P}_{x\sim\mathcal{D}}\left[ \forall \tau \in B(\alpha),\ d_1(\probmap(x+\tau),\probmap(x)) \leq \phi(\epsilon) \right] \leq 1-\gamma$. By inverting the inequality, one gets the expected result.
\end{proof}
\subsection{Proof of Proposition~\ref{prop:RobustTV-appendix}}

% \begin{proposition}
% [\cite{5605338}]
% \label{th:PolynomBoundTotalVar}
% Given two probability measures $\mu_1$ and $\mu_2$ on $(\mathcal{Y},\mathcal{F}_{\mathcal{Y}})$, 
% on has $$d_{KL}(\mu_1,\mu_2) \geq 2d_{TV}(\mu_1,\mu_2)^{2}+ \frac{4d_{TV}(\mu_1,\mu_2)^{4}}{9}$$
% \end{proposition}
% \begin{proposition}[\cite{Vajda1970}]
% \label{th:LogarithmicBoundTotalVar}
% Given two probability measures $\mu_1$ and $\mu_2$ on $(\mathcal{Y},\mathcal{F}_{\mathcal{Y}})$, 
% on has $$d_{KL}(\mu_1,\mu_2) \geq \log\left(\frac{1 + d_{TV}(\mu_1,\mu_2)}{1 - d_{TV}(\mu_1,\mu_2)} \right) - \frac{2d_{TV}(\mu_1,\mu_2)}{1 + d_{TV}(\mu_1,\mu_2)}$$
% \end{proposition}

\begin{proposition}[Renyi-robustness implies TV-robustness]
\label{prop:RobustTV-appendix}
Let $\probmap$ be a probabilistic mapping, then for all $\lambda\geq1$:
$$\probmap \text{ is }  d_{R,\lambda}\text{-}(\alpha, \epsilon, \gamma)\text{-robust} \implies \probmap \text{ is } d_{TV}\text{-}(\alpha, \epsilon', \gamma)\text{-robust}$$
$$\textnormal{ with } \epsilon' = \min \left(\frac{3}{2}\left(\sqrt{1 + \frac{4\epsilon}{9}} - 1\right)^{1/2}, \frac{\exp(\epsilon +1) -1}{\exp(\epsilon +1) +1}\right).$$

%And $\forall\lambda\in(0,1)$:
%$$u \text{ is }  d_{R,\lambda}\text{-}(\alpha, \epsilon, \gamma)\text{-robust} \implies u \text{ is } d_{TV}\text{-}(\alpha,\epsilon' , \gamma)\text{-robust} \textnormal{, with } \epsilon'=\sqrt{\frac{2\epsilon}{\alpha}}.$$

\end{proposition}

\begin{proof}
Given two probability measures $\mu_1$ and $\mu_2$ on $(\mathcal{Y},\mathcal{F}_{\mathcal{Y}})$, and $\lambda >0$ one wants to find a bound on $d_{TV}(\mu_1,\mu_2)$ as a functional of $d_{R,\lambda}(\mu_1,\mu_2)$. 

Thanks to~\cite{5605338}, one has $$ d_{KL}(\mu_1,\mu_2) \geq 2d_{TV}(\mu_1,\mu_2)^{2}+ \frac{4d_{TV}(\mu_1,\mu_2)^{4}}{9}. $$ From which it follows that $$d_{TV}(\mu_1,\mu_2)^{2} \leq \frac{9}{4}\left(\sqrt{1 + \frac{4d_{KL}(\mu_1,\mu_2)}{9}} - 1\right)$$
One thus finally gets: 
$$d_{TV}(\mu_1,\mu_2) \leq \frac{3}{2}\left(\sqrt{1 + \frac{4d_{KL}(\mu_1,\mu_2)}{9}} - 1\right)^{1/2}$$
Moreover, using inequality from~\cite{Vajda1970}, one gets:
$$d_{KL}(\mu_1,\mu_2) \geq \log\left(\frac{1 + d_{TV}(\mu_1,\mu_2)}{1 - d_{TV}(\mu_1,\mu_2)} \right) - \frac{2d_{TV}(\mu_1,\mu_2)}{1 + d_{TV}(\mu_1,\mu_2)}$$
For the sake of simplicity, since the second part of the right hand side of the equation is non increasing given $d_{TV}(\mu_1,\mu_2)$, and since $0\leq d_{TV}(\mu_1,\mu_2)\leq 1$  one gets:
$$d_{KL}(\mu_1,\mu_2) +1 \geq \log\left(\frac{1 + d_{TV}(\mu_1,\mu_2)}{1 - d_{TV}(\mu_1,\mu_2)} \right) $$
Hence, one gets:
$$ \frac{\exp(d_{KL}(\mu_1,\mu_2) +1) -1}{\exp(d_{KL}(\mu_1,\mu_2) +1) +1} \geq  d_{TV}(\mu_1,\mu_2) $$
By combining the two results, one obtains: $$ d_{TV}(\mu_1,\mu_2) \leq \min \left(\frac{3}{2}\left(\sqrt{1 + \frac{4 d_{KL}(\mu_1,\mu_2)}{9}} - 1\right)^{1/2}, \frac{\exp(d_{KL}(\mu_1,\mu_2) +1) -1}{\exp(d_{KL}(\mu_1,\mu_2) +1) +1} \right).$$
To conclude for $\lambda>1$ it suffices to use Lemma~\ref{th::PropimpliesRobustness-appendix}, and the monotony of Renyi divergence regarding $\lambda$.
%Analogously for $\lambda <1$ one can use the generalized Pinsker inequality to get the expected result.
\end{proof}

\subsection{Proof of Proposition~\ref{prop::postprocessing-appendix}}

\begin{proposition}[Data processing inequality]
\label{prop::postprocessing-appendix} 
Let us consider a probabilistic mapping $\probmap:\mathcal{X}\rightarrow\mathcal{P}(\mathcal{Y})$. Let us also denote $\rho:\mathcal{Y}\rightarrow\mathcal{Y}'$ a deterministic function.
If $U\sim \probmap(x)$ then the probability measure $M'(x)$ s.t. $\rho(U) \sim M'(x)$ defines a probabilistic mapping $M':\mathcal{X}\rightarrow\mathcal{P}(\mathcal{Y}')$.

For any $\lambda>1$ if $\probmap$ is $d_{R,\lambda}$-$(\alpha,\epsilon,\gamma)$ robust then $M'$ is also is $d_{R,\lambda}$-$(\alpha,\epsilon,\gamma)$ robust.
\end{proposition}

\begin{proof}
Let us consider $\probmap$ a $d_{R,\lambda}$-$(\alpha,\epsilon,\gamma)$ robust algorithm. Let us also take $x\in \mathcal{X}$, and $\tau \in B(\alpha)$.
Without loss of generality, we consider that $\probmap(x)$, and $\probmap(x+\tau)$ are dominated by the same measure $\mu$.
Finally let us take $\rho$ a measurable mapping from $(\mathcal{Y},\mathcal{F}_{\mathcal{Y}})$ to $(\mathcal{Y'},\mathcal{F}_{\mathcal{Y'}})$. For the sake of readability we denote $\probmap(x) =\nu_1$ and $M(x +\tau)=\nu_2$ (therefore $M'(x)=\rho \#\nu_1$, and $M'(x +\tau)=\rho \#\nu_2$).

Since $\mu >> \nu_1, \nu_2$, one has $\rho \# \mu >> \rho \# \nu_1, \rho \#\nu_2$. Hence one has 
 \begin{align*}
     d_{R,\lambda}(\rho \# \nu_1,\rho \# \nu_2) &= \frac{1}{\lambda - 1} \log \int_{\mathcal{Y}} \left(\frac{d\rho \# \nu_1}{d \rho \# \mu}\right)^{\lambda} \left(\frac{d\rho \# \nu_2}{d \rho \# \mu}\right)^{1-\lambda} d\rho \# \mu = \frac{1}{\lambda - 1} \log \int_{\mathcal{Y}} \left(\frac{d\rho \# \nu_1}{d \rho \# \nu_2}\right)^{\lambda} d\rho \# \nu_2\\
     \intertext{ Simply using the transfer theorem, one gets}
      d_{R,\lambda}(\rho \# \nu_1,\rho \# \nu_2) &= \frac{1}{\lambda - 1} \log \int_{\mathcal{Y'}} \left(\frac{d\rho \# \nu_1}{d \rho \# \nu_2} \circ \rho \right)^{\lambda} d\nu_2\\
     \intertext{Since $\left(\frac{d\rho \# \nu_1}{d \rho \# \nu_2} \circ \rho \right)= \mathbb{E}\left(\frac{d\nu_{1}}
    {d\nu_{2}}\big |\rho^{-1}(\mathcal{F}_{\mathcal{Y}})\right)$ one easily gets the following:}
    d_{R,\lambda}(\rho \# \nu_1,\rho \# \nu_2) & = \frac{1}{\lambda - 1} \log \int_{\mathcal{Y'}} \left(\frac{d\rho \# \nu_1}{d \rho \# \nu_2} \circ \rho \right)^{\lambda} d\nu_2 =  \frac{1}{\lambda - 1} \log \int_{\mathcal{Y'}} \mathbb{E}\left(\frac{d\nu_{1}}
    {d\nu_{2}}\big |\rho^{-1}(\mathcal{F}_{\mathcal{Y}})\right)^{\lambda} d\nu_{2}\\
    \intertext{Finally, by using the Jensen inequality, and the property of the conditional expectation, one has}
    d_{R,\lambda}(\rho \# \nu_1,\rho \# \nu_2)& \leq \frac{1}{\lambda - 1} \log \int_{\mathcal{Y'}} \mathbb{E}\left(\frac{d\nu_{1}}
    {d\nu_{2}}^{\lambda}\big |\rho^{-1}(\mathcal{F}_{\mathcal{Y}})\right) d\nu_{2}= \frac{1}{\lambda - 1} \log \int_{\mathcal{Y'}} \frac{d\nu_{1}}
    {d\nu_{2}}^{\lambda} d\nu_{2} = d_{R,\lambda}(\nu_{1},\nu_{2}).
 \end{align*}
\end{proof}

\subsection{Proof of Theorem~\ref{thm:netrob-appendix}}

\begin{lemma}
\label{thm:exprob-appendix}

Let $\psi:\ \mathcal{X} \rightarrow \mathbb{R}^{n}$ be a mapping. For any $\alpha\geq0$ and for any norms $||.||_A$ and $||.||_B$, one can define $\Delta^{A,B}_{\alpha}(\psi):=\sup\limits_{ x,y \in \mathcal{X}, ||x-y||_A \leq \alpha} || \psi(x) - \psi(y) ||_B$. Let $X$ be a random variable. We denote $\probmap(x)$ the probability measure of the random variable $\psi(x)+X$. 
\begin{itemize}
    \item If $X\sim E_{F}(\theta,t,k)$  where $t$ and $k$ have non-decreasing modulus of continuity $\omega_t$ and $\omega_k$. 

Then for any $\alpha \geq 0$, $\probmap$ defines a probabilistic mapping that is $d_{R,\lambda}$-$(\alpha,\epsilon)$ robust
with $\epsilon = ||\theta||_2 \omega^{B,2}_t(\Delta^{A,B}_{\alpha}(\phi)) +\omega_k^{B,1}(\Delta^{A,B}_{\alpha}(\phi)) $ where $||.||_2$ is the norm corresponding to the scalar product in the definition of the exponential family density function and $||.||_1$ is here the absolute value on $\mathbb{R}$.
%Although the Gaussian distribution does not satisfy the modulus of continuity constraint on $t$, we still have robustness for Gaussian noise injection. Let
\item If $X$ is a centered Gaussian random variable with a non degenerated matrix parameter $\Sigma$. Then for any $\alpha \geq 0$, $\probmap$ defines a probabilistic mapping that is $d_{R,\lambda}$-$(\alpha,\epsilon)$ robust
with $ \epsilon = \frac{\lambda \Delta^{A,2}_{\alpha}(\phi)^2 }{2 \sigma_{min}(\Sigma) } $ where $||.||_2$ is the canonical Euclidean norm on $\mathbb{R}^n$.
\end{itemize}
\end{lemma}

\begin{proof}
Let us consider $\probmap$ the probabilistic mapping constructed from noise injections respectively drawn from 1) an exponential family with non-decreasing modulus of continuity, or 2) a non degenerate Gaussian. Let us take $x\in \mathcal{X}$, and $\tau \in B(\alpha)$.
Without loss of generality, we consider that $\probmap(x)$, and $\probmap(x+\tau)$ are dominated by the same measure $\mu$. Let us also denote, $p_F$ the Radon-Nikodym derivative of the noise drawn in 1) with respect to $\mu$, $p_G$ the Radon-Nikodym derivative of the noise drawn or in 2) with respect to $\mu$ and $ \delta_a$ the Dirac function mapping any element to $1$ if it equals $a$ and to 0 otherwise.

 1)\begin{align*}
    d_{R,\lambda}\left(\probmap(x),\probmap(x+\tau)
    \right)&=d_{R,\lambda}\left(\nu * \delta_{\psi(x)},\nu * \delta_{\psi(x+\tau)}  \right)\\
    &\leq d_{R,\infty}\left(\nu * \delta_{\psi(x)},\nu * \delta_{\psi(x+\tau)}  \right)\\
    &=\log \sup\limits_{z \in \mathbb{R}^{n}} \frac{(p_F * \delta_{\psi(x)})(z)}{(p_F * \delta_{\psi(x+\tau)})(z)}\\
    &=\log \sup\limits_{z \in \mathbb{R}^{n}} \exp(<t(z-\psi(x)) - t(z-\psi(x + \tau)),\theta> \\
    &+ k(z -\psi(x))-k(z-\psi(x +\tau))) \\
    &\leq  \sup\limits_{z \in \mathbb{R}^{n}} ||\theta||_2 ||t(z-\psi(x)) - t(z-\psi(x + \tau))||_2 +|k(z-\psi(x)) - k(z-\psi(x + \tau))|\\
    &\leq ||\theta||_2 \omega^{B,2}_t(||\psi(x+\tau)-\psi(x)||_B) +\omega^{B,1}_k(||\psi(x+\tau)-\psi(x)||_B)\\
    &\leq ||\theta||_2 \omega^{B,2}_t(\Delta^{A,B}_{\alpha}(\psi)) +\omega^{B,1}_k(\Delta^{A,B}_{\alpha}(\psi))
\end{align*}
2)  
\begin{align*}
    & d_{R,\lambda}\left( \probmap(x),\probmap(x+\tau) \right) = \frac{1}{\lambda -1}\log \int_{\mathbb{R}^{n}} \left(p_G * \delta_{\psi(x)}\right)^{\lambda}\times \left(p_G* \delta_{\psi(x+\tau)}\right)^{1-\lambda} d\mu \\\\
   &= \frac{1}{\lambda -1}\log \int_{\mathbb{R}^{n}} \frac{\exp\left\{-1/2\left(\lambda\left(z - \psi(x)\right)^\top\Sigma^{-1}\left(z - \psi(x)\right) + (1-\lambda)\left(z - \psi(x+\tau)\right)^\top\Sigma^{-1}\left(z - \psi(x+\tau)\right)\right)\right\}}{(2 \pi)^{n/2}|\Sigma|^{1/2}}  dz \\
   &= \frac{-\left(\lambda \psi(x)^\top\Sigma^{-1}\psi(x) + (1-\lambda) \psi(x+\tau)^\top\Sigma^{-1}\psi(x+\tau) - \left(\lambda \psi(x) + (1- \lambda) \psi(x+\tau)\right)^\top\Sigma^{-1}\left(\lambda \psi(x) + (1- \lambda) \psi(x+\tau)\right)\right)}{2\lambda -2} \\
    &= \frac{\lambda^{2} - \lambda}{2(\lambda -1)} (\psi(x) - \psi(x+\tau))^\top \Sigma^{-1}(\psi(x) - \psi(x+\tau))\\
    &\leq \frac{\lambda}{2} \sigma_{max}\left(\Sigma^{-1}\right) ||(\psi(x) - \psi(x+\tau))||_2^2 \leq \frac{\lambda \Delta^{A,2}_{\alpha}(\psi)^2 }{ 2 \sigma_{min}(\Sigma)}.
\end{align*}
\end{proof}

\begin{theorem}[Exponential family ensures robustness]
\label{thm:netrob-appendix}
Let us denote $\mathcal{N}_{X}^i(.)=\phi^n\circ...\circ\phi^{i+1}(\mathcal{N}_{|i}(.)+X)$ with $X$ a random variable. Let us also consider $||.||_{A}$, and $||.||_{B}$ two arbitrary norms respectively on $\mathcal{X}$ and on the output space of $\mathcal{N}_{X}^i$.

\begin{itemize}
    \item If $X\sim E_{F}(\theta,t,k)$  where $t$ and $k$ have non-decreasing modulus of continuity $\omega_t$ and $\omega_k$. Then for any $\alpha \geq 0$, $\mathcal{N}_{X}^i(.)$ defines a probabilistic mapping that is $d_{R,\lambda}$-$(\alpha,\epsilon)$ robust with $\epsilon = ||\theta||_2 \omega^{B,2}_t(\Delta^{A,B}_{\alpha}(\phi)) +\omega_k^{B,1}(\Delta^{A,B}_{\alpha}(\phi)) $ where $||.||_2$ is the norm corresponding to the scalar product in the definition of the exponential family density function and $||.||_1$ is here the absolute value on $\mathbb{R}$. The notion of continuity modulus is defined in the preamble of this supplementary material.
    
%Although the Gaussian distribution does not satisfy the modulus of continuity constraint on $t$, we still have robustness for Gaussian noise injection. Let
\item If $X$ is a centered Gaussian\footnote{Although the Gaussian distribution belongs to the exponential family, it does not satisfy the modulus of continuity constraint on $t$ and its robustness has to be proved differently.} random variable with a non degenerated matrix parameter $\Sigma$. Then for any $\alpha \geq 0$, $\mathcal{N}_{X}^i(.)$ defines a probabilistic mapping that is $d_{R,\lambda}$-$(\alpha,\epsilon)$ robust
with $ \epsilon = \frac{\lambda \Delta^{A,2}_{\alpha}(\phi)^2 }{2 \sigma_{min}(\Sigma) } $ where $||.||_2$ is the canonical Euclidean norm on $\mathbb{R}^n$.
\end{itemize}
\end{theorem}

\begin{proof}
This theorem is a direct consequence of Lemma~\ref{thm:exprob-appendix} and Proposition~\ref{prop::postprocessing-appendix}. By applying Lemma~\ref{thm:exprob-appendix} to $\psi=\mathcal{N}_{|i}$ and Proposition~\ref{prop::postprocessing-appendix} to $\rho=\phi^n\circ...\circ\phi^{i+1}$, we immediately get the result.
\end{proof}

\subsection{Proof of Theorem~\ref{thm:bound-appendix}}

\begin{theorem}[Adversarial generalization gap bound in the randomized setting]

\label{thm:bound-appendix}
Let $\probmap$ be the probabilistic mapping at hand. Let suppose that  $\probmap$ is $d_{R,\lambda}$-$(\alpha,\epsilon)$ robust for some $\lambda\geq1$ then:

$$|\advRisk(\probmap)-\Risk(\probmap)|\leq 1-e^{-\epsilon}\mathbb{E}_x\left[e^{-H(\probmap(x))}\right]$$
where $H$ is the Shannon entropy: $H(p)=-\sum_i p_i \log(p_i)$
\end{theorem}

\begin{proof}
Let $\probmap$ be a randomized network with a noise $X$ injected at layer $i$. We have:
\begin{align*}
\left|\advRisk(\probmap)-\Risk(\probmap)\right|  &= \left|\mathbb{E}_{(x,y)}\left[ \sup_{\tau/||\tau||\leq\alpha} \mathbb{E}_{y'\sim \probmap(x+\tau)}\left[ \mathds{1}\left( y_1\neq y\right) \right]-  \mathbb{E}_{y_2\sim \probmap(x)}\left[ \mathds{1}\left(y'\neq y\right)\right]\right]\right| \\
&=\left|\mathbb{E}_{(x,y)}\left[ \sup_{\tau/||\tau||\leq\alpha} \mathbb{E}_{y_1\sim \probmap(x+\tau),y_2\sim \probmap(x)}\left[ \mathds{1}\left(y_1\neq y\right)-  \mathds{1}\left(y_2\neq y\right) \right]\right]\right|\\
&\leq\mathbb{E}_{(x,y)}\left[ \sup_{\tau/||\tau||\leq\alpha} \mathbb{E}_{y_1\sim \probmap(x+\tau),y_2\sim \probmap(x)}\left[\left| \mathds{1}\left(y_1\neq y\right)-  \mathds{1}\left(y_2\neq y\right)\right|\right]\right]\\
&\leq\mathbb{E}_{(x,y)}\left[ \sup_{\tau/||\tau||\leq\alpha} \mathbb{E}_{y_1\sim \probmap(x+\tau),y_2\sim \probmap(x)}\left[ \mathds{1}\left(y_1\neq y_2\right)\right]\right]\\
&=\mathbb{E}_{(x,y)}\left[\sup_{\tau/||\tau||\leq\alpha}\mathbb{P}_{y_1\sim \probmap(x+\tau),y_2\sim \probmap(x)}(y_1\neq y_2)\right]
\end{align*}

For two discrete random independent variables of law $P=(p_1,...,p_K)$ and $Q=(q_1,...,q_K)$, thanks to Jensen's inequality: 
$$\mathbb{P}(P=Q)=\sum_{i=1}^K p_i q_i \geq \exp{(\sum_{i=1}^K p_i \log q_i)}=\exp{(-d_{KL}(P,Q)-H(P))}$$

Then we have:
\begin{align*}
\mathbb{E}_{(x,y)}\left[\sup_{\tau/||\tau||\leq\alpha}\mathbb{P}_{y_1\sim \probmap(x+\tau),y_2\sim \probmap(x)}(y_1\neq y_2)\right] &\leq\mathbb{E}_{(x,y)}\left[\sup_{\tau/||\tau||\leq\alpha}1-e^{-d_{KL}(\probmap(x),\probmap(x+\tau))-H(\probmap(x))} \right]\\
&\leq \mathbb{E}_{(x,y)}\left[1-e^{-\epsilon-H(\probmap(x))} \right]\\
&=1-e^{-\epsilon}\mathbb{E}_x\left[e^{-H(\probmap(x))}\right]\\
\end{align*}

\end{proof}

\section{Additional results and discussions}
In this section, we give some additional results on both the strength of the Renyi-divergence and a bound on the generalization gap for TV-distance.

\subsection{About Renyi divergence}
In the main submission, we chose to use the Renyi-robustness as the principled measure of robustness. Since Renyi-divergence is a good surrogate for the trivial distance (which is a generalization of the $0-1$-loss for probabilistic mappings), we supported this statement by showing that Renyi-divergence is stronger than TV-distance.
In this section, we extend this result to most of the classical divergences used in Machine Learning and show that Renyi-divergence is stronger than all of them.   

Let us consider an output space $\mathcal{Y}$, $\mathcal{F}_{\mathcal{Y}}$ a $\sigma$-$ algebra$ over $\mathcal{Y}$, and $\mu_1,\mu_2,\nu$ three measures on $(\mathcal{Y},\mathcal{F}_{\mathcal{Y}})$, with $\mu_1,\mu_2$ in the set of probability measures over $(\mathcal{Y},\mathcal{F}_{\mathcal{Y}})$ denoted $\mathcal{P}(\mathcal{Y})$.
One has $\nu >> \mu_1,\mu_2$ and one denotes $g_1$ and $g_2$ the Radon-Nikodym derivatives with respect to $\nu$.

\textbf{The Separation distance:}
$$d_{S}(\mu_1,\mu_2):=\underset{\{z\} \in \mathcal{F}_{\mathcal{Y}}}{\sup}  1 - \frac{\mu_1(\{z\})}{\mu_2(\{z\})}. $$

\textbf{The Hellinger distance:} 
$$ d_{H}(\mu_1,\mu_2):= \left[ \int_{\mathcal{Y}} \left(\sqrt{g_1} - \sqrt{g_2} \right)^{2} d \nu \right]^{1/2}.$$

\textbf{The Prokhorov metric:} 
$$d_{P}(\mu_1,\mu_2):=\inf\left\{ \zeta >0: \mu_1(B) \leq \mu_2(B^{\zeta}) + \zeta \textnormal{ for all Borel sets } B\right\}
\textnormal{ where } B^{\zeta} = \{x : \inf\limits_{y \in B} d_{\mathcal{Y'}}(x,y)\leq \zeta \}.$$

\textbf{The Discrepancy metric:}
$$d_{D}(\mu_1,\mu_2):=\underset{\textnormal{all closed balls B}}{\sup} \left| \mu_1(B) - \mu_2(B) \right|.$$

\begin{lemma}\label{lemma:separation-appendix}
Given two probability measures $\mu_1$ and $\mu_2$ on $(\mathcal{Y},\mathcal{F}_{\mathcal{Y}})$  the Separation metric and the Renyi divergence satisfy the following relation: $d_S(\mu_1,\mu_2) \leq d_{R,\infty}(\mu_1,\mu_2)$
\end{lemma}

\begin{proof}
The function $x: \to 1-x - \left|\ln(x)\right|$ is negative on $\mathbb{R}$, therefore for any $\{z\} \in \mathcal{Y}$ one has $1-\frac{\mu_1(\{z\})}{\mu_2(\{z\})} \leq \left|\ln\frac{\mu_1(\{z\})}{\mu_2(\{z\})}\right|$, hence $\sup_{\{z\} \in \mathcal{F}_{\mathcal{Y}}} 1-\frac{\mu_1(\{z\})}{\mu_2(\{z\})} \leq \sup_{\{z\} \in \mathcal{F}_{\mathcal{Y}}}\left|\ln\frac{\mu_1(\{z\})}{\mu_2(\{z\})}\right| \leq \sup_{Z \in \mathcal{F}_{\mathcal{Y}}}\left|\ln\frac{\mu_1(Z)}{\mu_2(Z)}\right| = d_{R,\infty}(\mu_1,\mu_2)$
\end{proof}

% \begin{proposition}[\cite{AGibbsMetrics2002}]
% \label{th:LinkWassersteinTV}
% Given two probability measures $\mu_1$ and $\mu_2$ on $(\mathcal{Y},\mathcal{F}_{\mathcal{Y}})$,  the Wasserstein metric and the Total Variation distance satisfy the following relation: $d_W(\mu_1,\mu_2) \leq diam(\mathcal{Y}) d_{TV}(\mu_1,\mu_2)$.
% \end{proposition}

% \begin{proposition}[\cite{AGibbsMetrics2002}]
% \label{th:LinkHellinger}
% Given two probability measures $\mu_1$ and $\mu_2$ on $(\mathcal{Y},\mathcal{F}_{\mathcal{Y}})$  the Hellinger distance and the Kullback Leibler divergence satisfy the following relation: $d_H(\mu_1,\mu_2)^{2} \leq d_{KL}(\mu_1,\mu_2).$
% \end{proposition}

% \begin{proposition}[\cite{AGibbsMetrics2002}]
% \label{th:LinkProkhorovTV}
% Given two probability measures $\mu_1$ and $\mu_2$ on $(\mathcal{Y},\mathcal{F}_{\mathcal{Y}})$  the Prokhorov metric and the total variation distance satisfy the following relation: $d_D(\mu_1,\mu_2) \leq d_{TV}(\mu_1,\mu_2).$
% \end{proposition}

% \begin{proposition}[\cite{huber2011robust}] %p34
% Given two probability measures $\mu_1$ and $\mu_2$ on $(\mathcal{Y},\mathcal{F}_{\mathcal{Y}})$  the Discrepancy metric and the total variation distance satisfy the following relation: $d_P(\mu_1,\mu_2) \leq d_{TV}(\mu_1,\mu_2).$
% \end{proposition}

\begin{theorem}
Let $\probmap$ be the probabilistic mapping, then for all $\lambda >1$ if $\probmap \text{ is }  d_{R,\lambda}\text{-}(\alpha, \epsilon, \gamma)\text{-robust}$ 
the following assertions holds:
\begin{itemize}
\item[(1)] $\probmap$ is $d_{H}\text{-}(\alpha, \sqrt{\epsilon}, \gamma)\text{-robust}.$
\item[(2)] $\probmap$ is $d_{P}\text{-}(\alpha, \epsilon', \gamma)\text{-robust} \textbf{ and } d_{D}\text{-}(\alpha, \epsilon', \gamma)\text{-robust}$, for $\epsilon' = \min \left(\frac{3}{2}\left(\sqrt{1 + \frac{4\epsilon}{9}} - 1\right)^{1/2}, \frac{\exp(\epsilon +1) -1}{\exp(\epsilon +1) +1}\right).$
\item[(3)] $\probmap \text{ is } d_{W}\text{-}(\alpha, \epsilon', \gamma)\text{-robust} \text{ with }\epsilon' = \min \left(\frac{3}{2 diam(\mathcal{Y})}\left(\sqrt{1 + \frac{4\epsilon}{9}} - 1\right)^{1/2}, \frac{\exp(\epsilon +1) -1}{ diam(\mathcal{Y}) (\exp(\epsilon +1) +1)}\right)$.
\item[(4)] if $\lambda =\infty$, $ \probmap \text{ is } d_{S}\text{-}(\alpha, \epsilon, \gamma)\text{-robust}.$
\end{itemize}
\end{theorem}

\begin{proof}
\begin{itemize}
\item[ ]
\item[(1)] The proof is a simple adaptation of Proposition~\ref{prop:RobustTV-appendix} using the inequality $d_H(\mu_1,\mu_2)^{2} \leq d_{KL}(\mu_1,\mu_2)$~\cite{AGibbsMetrics2002} and Lemma~\ref{th::PropimpliesRobustness-appendix}.\\
\item[(2)] Using the inequalities  $d_D(\mu_1,\mu_2) \leq d_{TV}(\mu_1,\mu_2)$~\cite{AGibbsMetrics2002} and  $d_P(\mu_1,\mu_2) \leq d_{TV}(\mu_1,\mu_2)$~\cite{huber2011robust}, the proof is immediate, using Theorem~\ref{th::PropimpliesRobustness-appendix} and Lemma~\ref{th::PropimpliesRobustness-appendix}.
\item[(3)] The proof is adapted from Proposition~\ref{prop:RobustTV-appendix} using the inequality $d_W(\mu_1,\mu_2) \leq diam(\mathcal{Y}) d_{TV}(\mu_1,\mu_2)$~\cite{AGibbsMetrics2002}.
\item[(4)]The result is a straightforward application of Lemma~\ref{lemma:separation-appendix}, and Lemma~\ref{th::PropimpliesRobustness-appendix}.
\end{itemize}
\end{proof}
%\end{onehalfspace}

\subsection{Generalization gap with TV-robustness}
In the main paper, we give a bound on the generalization gap based on the Renyi-robustness. We extend this result to the TV-robustness, highlighting the fact that generalization gap could be derived from any of the above divergences.

\begin{theorem}

\label{thm:boundTV-appendix}
Let $\probmap$ be the probabilistic mapping at hand. Let us suppose that  $\probmap$ is $d_{TV}$-$(\alpha,\epsilon)$ robust then:
$$|\advRisk(\probmap)-\Risk(\probmap)|\leq 1-(\mathbb{E}_x\left[e^{-H_c(\probmap(x))}\right]-\epsilon)$$
where $H_c$ is the collision entropy: $H_c(p)=-\log(\sum_i p_i^2)$
\end{theorem}
\begin{proof}
For two discrete random independent variables of law $P=(p_1,...,p_K)$ and $Q=(q_1,...,q_K)$, thanks to Jensen's inequality: 
$$\mathbb{P}(P=Q)=\sum_{i=1}^K p_i q_i=\sum_{i=1}^K p_i^2-\sum p_i(p_i-q_i)\geq e^{-H_c(P)}-d_{TV}(P,Q)$$
because, for any $i \in [K]$, $p_i-q_i\leq d_{TV}(P,Q)$.

Then, the proof is a simple adaptation of the model of proof from Theorem~\ref{thm:bound-appendix}.
\end{proof}

\section{Additional empirical evaluation}
Due to space limitations, we had to defer the thorough description of our experimental setup and the results of some additional experiments.
\subsection{Architectures \& Hyper-parameters}

We conduct experiments with 3 different dataset:
\begin{itemize}
    \item CIFAR-10 and CIFAR-100 datasets, which are composed of 50K training samples, $10000$ test samples and respectively 10 and 100 different classes. Images are trained and evaluated with a resolution of 32 by 32 pixels. 
    \item ImageNet dataset, which is composed of $\sim1.2$M training examples, $50K$ test samples and $1000$ classes. Images are trained and evaluated with a resolution of 299 by 299 pixels. 
\end{itemize}

For CIFAR-10 and CIFAR-100 \cite{krizhevsky2009learning}, we used a Wide ResNet architecture \cite{ZagoruykoK16} which is a variant of the ResNet model from \cite{he2016deep}. We used 28 layers with a widen factor of 10. We trained all the networks for 200 epochs, a batch size of 400, dropout 0.3 and Leaky Relu activation with a slope on $\mathbb{R}^-$ of 0.1. We used the cross entropy loss with Momentum 0.9 and  a piecewise constant learning rate of 0.1, 0.02, 0.004 and 0.00008 after respectively 7500, 15000 and 20000 steps. The networks achieve for CIFAR10 and 100 a TOP-1 accuracy of 95.8\% and 79.1\% respectively on test images. 

For ImageNet \cite{imagenet_cvpr09}, we used an Inception ResNet v2 \cite{szegedy2017inception} which is the sate of the art architecture for this dataset and achieved a TOP-1 accuracy of 80\%. For the training of ImageNet, we used the same hyper parameters setting as the original implementation. We trained the network for 120 epochs with a batch size of 256, dropout 0.8, Relu as activation function. All evaluations were done with a single crop on the non-blacklisted subset of the validation set.

\subsection{Evaluation under attack}
We evaluate our models against the strongest possible attacks from the literature using different norms ($\ell_1$, $\ell_2$ and $\ell_\infty$) which are all optimization based attacks. On their guide to evaluate robustness, Carlini et al. \cite{carlini2019evaluating} proposed the three following attacks for each norm:

\paragraph{$\ell_2$ -- Carlini \& Wagner attack and $\ell_1$ -- ElasticNet attack}The $\ell_2$ Carlini \& Wagner attack ($C\&W$) introduced in~\cite{carlini2017towards} is formulated as:
$$\min_{x+r\in \mathcal{X} } c\times||r||_2+g(x+r)$$
where $g$ is a function such that $g(y)\geq 0$ iff $f(y)=l'$ with $l'$ the target class. The authors listed some $g$ functions. we choose the following one:
$$g(x)=\max(F_{k(x)}(x)-\max_{i\neq k(x)}(F_i(x)),-\kappa)$$
where $F$ is the softmax function and $\kappa$ a positive constant.

Instead of using box-constrained L-BFGS~\cite{Szegedy2013IntriguingPO} as in the original attack, the authors use instead a new variable for $x+r$:
$$x+r=\frac{1}{2} (\tanh(w)+1)$$ 
\medbreak
Then a binary search is performed to optimize the constant $c$ and ADAM or SGD for computing an optimal solution.

$\ell_1$ -- ElasticNet attack is an adaptation of $\ell_2$ C\&W attack where the objective is adaptive to $\ell_1$ perturbations:
$$\min_{x+r\in \mathcal{X} } c_1\times\norm{r}_1+c_2\times\norm{r}_2+g(x+r)$$

\paragraph{$\ell_\infty$ -- PGD attack.}
The PGD attack proposed by~\cite{madry2017towards} is a generalization of the iterative FGSM attack proposed in~\cite{kurakin2016adversarial}. The goal of the adversary is to solve the following problem: %type (2) problem
$$\argmax_{\norm{r}_p \leq \epsilon} \mathcal{L}(F_{\theta}(x+r),y) $$
In practice, the authors proposed an iterative method to compute a solution:
$$x^{t+1}=P_{x \oplus r}(x^t+\alpha \sign (\nabla_x\mathcal{L}(F_\theta(x^t),y)))$$
Where $x \oplus r$ is the Minkowski sum between $\{x\}$ and $\{r \text{~s.t.~} \norm{r}_p \leq \epsilon\}$, $\alpha$ a gradient step size, $P_S$ is the projection operator on $S$ and $x^0$ is randomly chosen in $x \oplus r$. 

\subsection{Detailed results on CIFAR-10 and CIFAR-100}

\begin{figure}[htb]
\caption{(a) Impact of the standard deviation of the injected noise on accuracy in a randomized model on CIFAR-100 dataset with a Wide ResNet architecture. (b) and (c) illustration of the guaranteed accuracy of different randomized models with Gaussian (b) and Laplace (c) noises given the norm of the adversarial perturbation.}
\centering
\subfigure[]{
    \includegraphics[width=.3\textwidth]{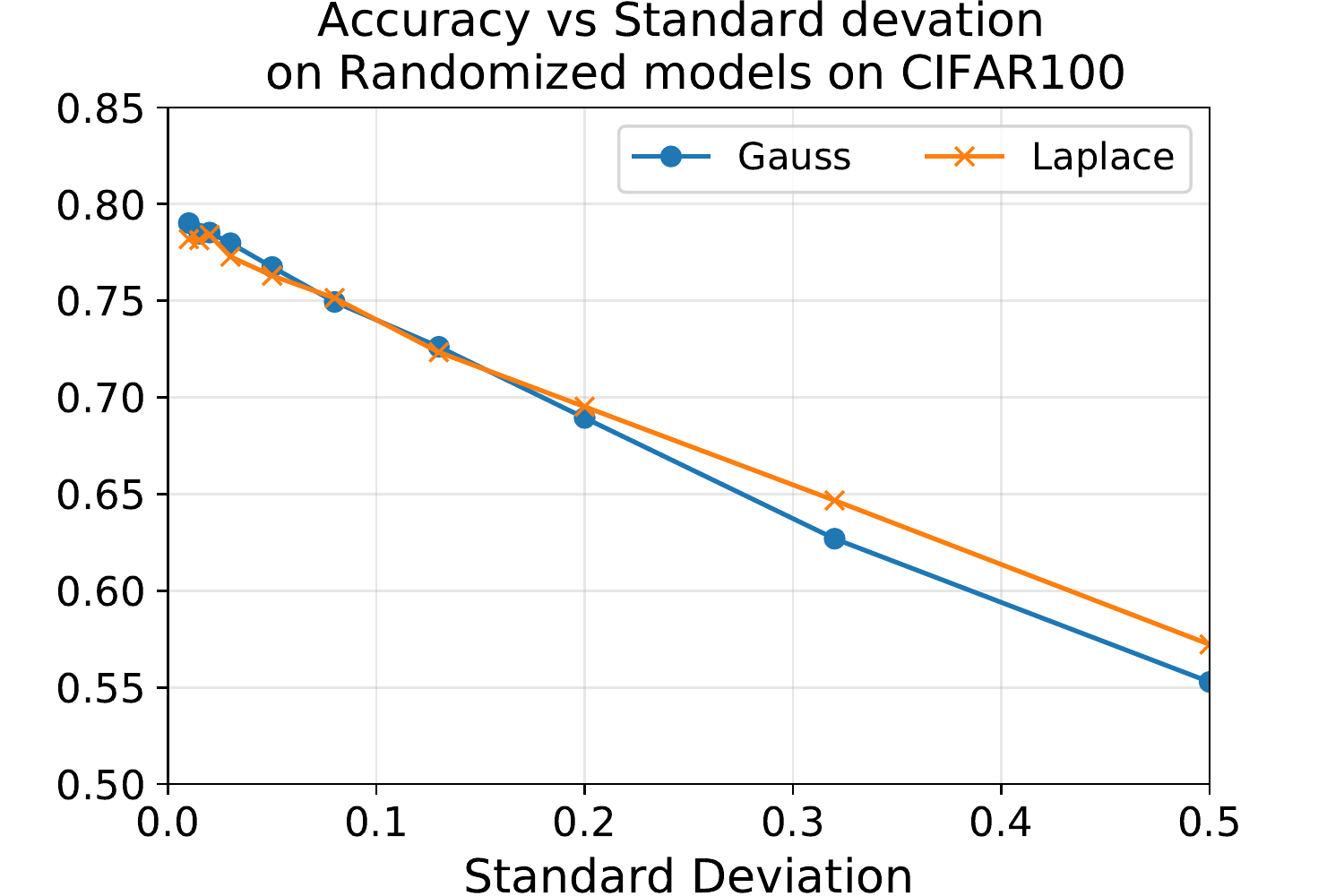}
    \label{fig:acc_sd_CIFAR100-appendix}
    }
\subfigure[]{
    \includegraphics[width=.3\textwidth]{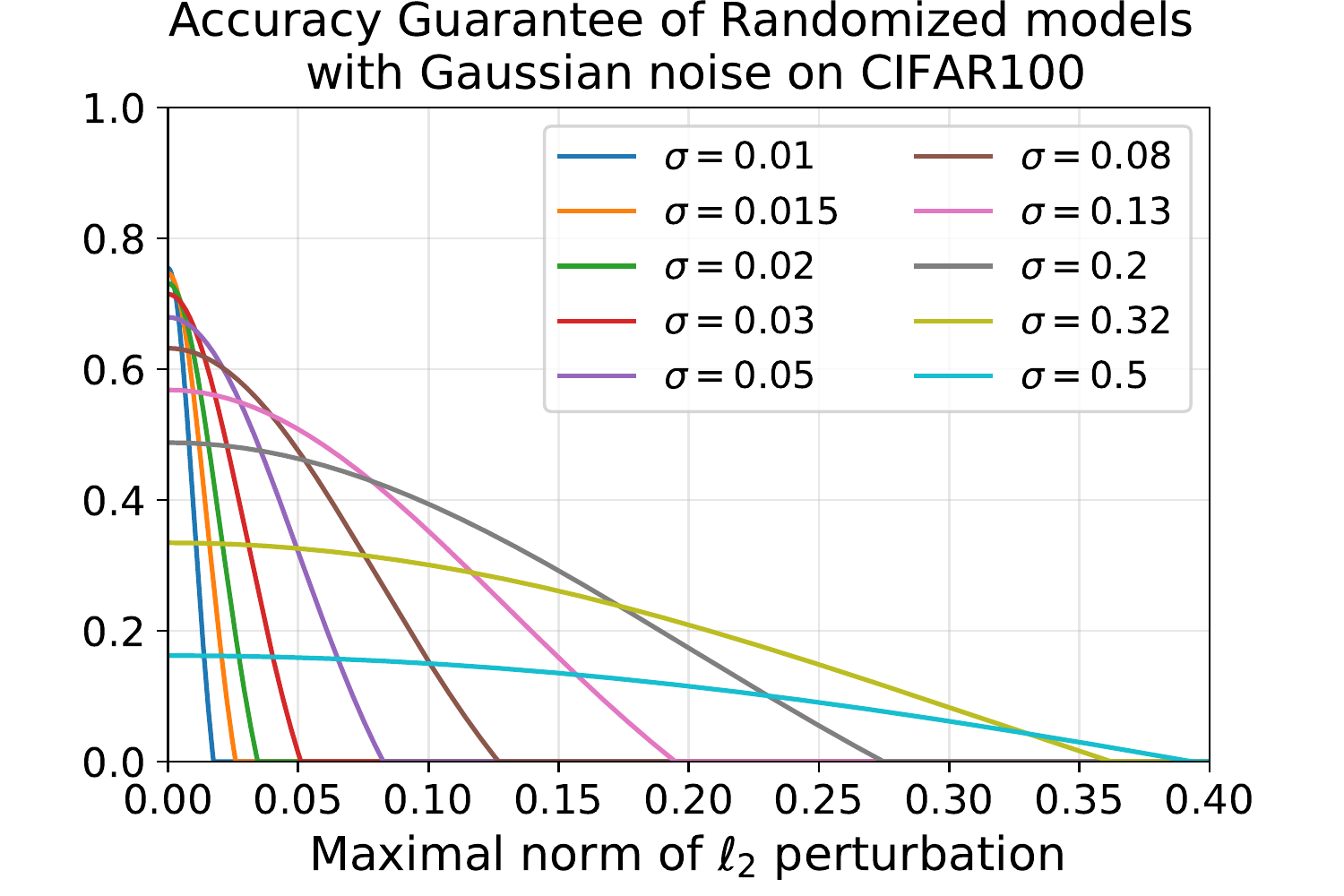}
    \label{fig:gauss_certif_CIFAR100-appendix}
    }
\subfigure[]{
    \includegraphics[width=.3\textwidth]{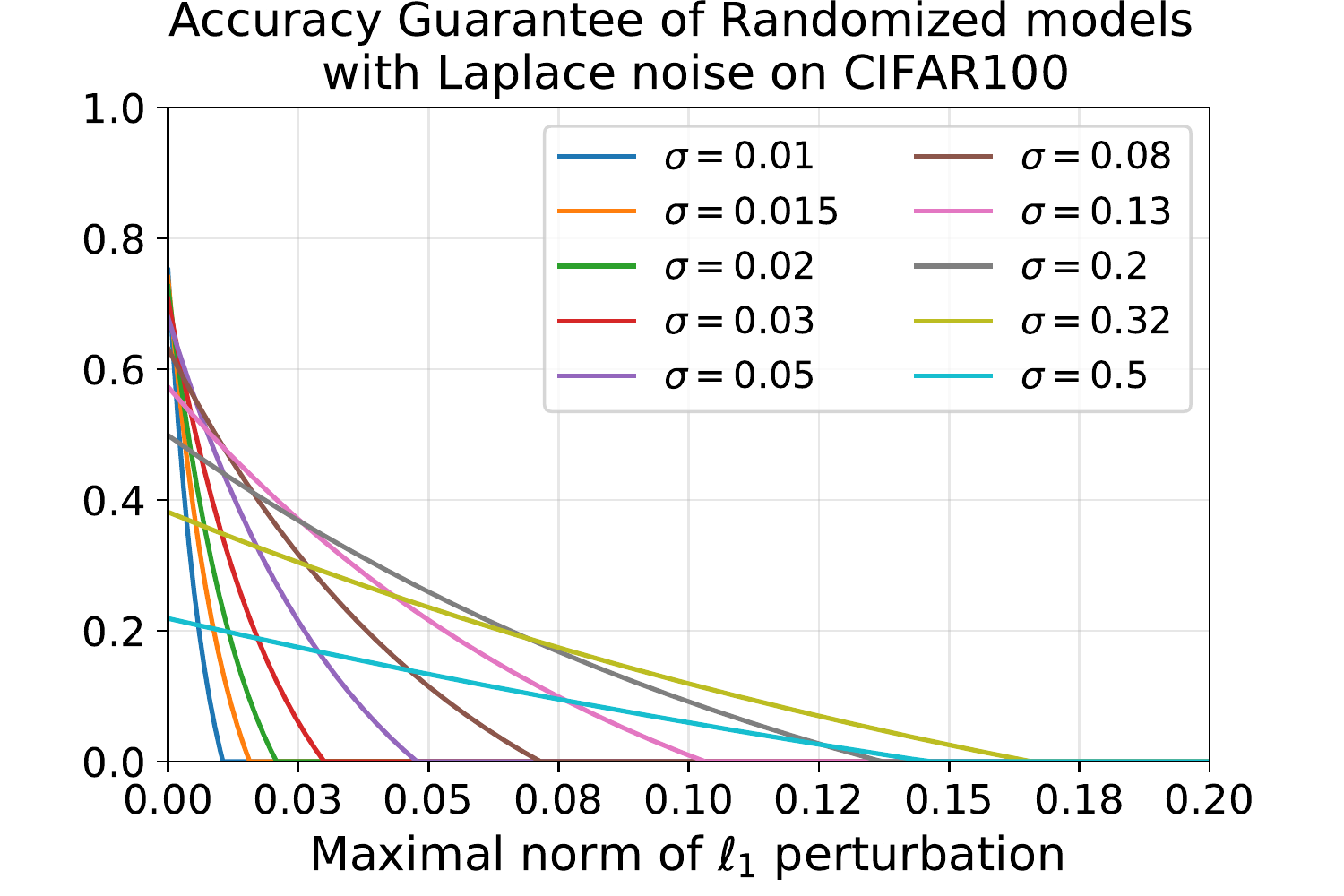}
    \label{fig:laplace_certif_CIFAR100-appendix}
    }
% \label{fig:cifar100_results}
\end{figure}

Figure~\ref{fig:acc_sd_CIFAR100-appendix} presents the trade-off accuracy versus intensity of noise for the CIFAR-100 dataset. As for CIFAR-10, we observe that the accuracy decreases from 0.79 with a small noise (0.01) to $\sim$0.55 with a higher noise (0.5). The Figures~\ref{fig:gauss_certif_CIFAR100-appendix} and \ref{fig:laplace_certif_CIFAR100-appendix} are coherent with the theoretical guarantee of accuracy (Theorem~\ref{thm:bound-appendix}) that the model can achieve under attack with a given perturbation and noise.

Table~\ref{tab:cifar10-appendix} and~\ref{tab:cifar100-appendix} summarize the results on the accuracy and accuracy under attack of CIFAR-10 and CIFAR-100 datasets with a Randomized Wide ResNet architecture given the standard deviation of the injected noise and the number of iterations of the attack. For PGD, we use an epsilon max of 0.06 and a step size of 0.006 for an input space of between -1 and +1. We show that injecting noise empirically helps defending neural networks against adversarial attacks.

\begin{table}[htb]
  \centering
  \tiny
  \caption{Accuracy and Accuracy under attack of CIFAR-10 dataset}
  \label{tab:cifar10-appendix}%
    \begin{tabular}{rcccccccccccc}
    \toprule
          & \multirow{2}[3]{*}{\textbf{Natural}} & \multicolumn{3}{c}{\textbf{$\ell_1$ -- EAD}} &       & \multicolumn{3}{c}{\textbf{$\ell_2$ -- C\&W}} &       & \multicolumn{3}{c}{\textbf{$\ell_\infty$ -- PGD}} \\
\cmidrule{3-5}\cmidrule{7-9}\cmidrule{11-13}          &       & \textbf{20} & \textbf{50} & \textbf{60} &       & \textbf{20} & \textbf{50} & \textbf{60} &       & \textbf{10} & \textbf{15} & \textbf{20} \\
    \textbf{Normal (Sd)} &       &       &       &       &       &       &       &       &       &       &       &  \\
    \midrule
    0.010 & 0.954 &       & 0.208 & 0.193 &       & 0.172 & 0.271 & 0.294 &       & 0.411 & 0.428 & 0.408 \\
    % 0.015 & 0.956 &       & 0.255 & 0.228 &       & 0.226 & 0.318 & 0.341 &       & 0.454 & 0.423 & 0.393 \\
    % 0.020 & 0.956 & 0.206 & 0.241 &       &       & 0.258 & 0.390 & 0.368 &       & 0.468 & 0.468 & 0.411 \\
    % 0.030 & 0.953 & 0.260 & 0.302 & 0.283 &       & 0.336 & 0.414 & 0.432 &       & 0.543 & 0.488 & 0.442 \\
    0.050 & 0.950 & 0.265 & 0.347 & 0.367 &       & 0.350 & 0.454 & 0.423 &       & 0.638 & 0.549 & 0.486 \\
    % 0.080 & 0.942 & 0.347 & 0.364 & 0.415 &       & 0.415 & 0.504 & 0.498 &       & 0.669 & 0.600 & 0.534 \\
    0.130 & 0.931 & 0.389 & 0.401 & 0.411 &       & 0.443 & 0.495 & 0.515 &       & 0.710 & 0.636 & 0.553 \\
    0.200 & 0.913 & 0.411 & 0.456 &       &       & 0.470 & 0.481 & 0.516 &       & \textbf{0.724} & 0.629 & 0.539 \\
    0.320 & 0.876 & 0.442 & 0.450 & 0.445 &       & 0.475 & \textbf{0.522} & 0.499 &       & 0.720 & \textbf{0.641} & 0.566 \\
    0.500 & 0.824 & \textbf{0.453} & \textbf{0.513} & \textbf{0.448} &       & \textbf{0.503} & 0.494 & \textbf{0.523} &       & 0.694 & 0.608 & \textbf{0.587} \\
          &       &       &       &       &       &       &       &       &       &       &       &  \\
    \textbf{Laplace (Sd)} &       &       &       &       &       &       &       &       &       &       &       &  \\
    \midrule
    0.010 & 0.955 & 0.167 & 0.190 & 0.208 &       & 0.184 & 0.279 & 0.313 &       & 0.474 & 0.423 & 0.389 \\
    % 0.015 & 0.956 & 0.213 & 0.242 & 0.260 &       & 0.242 & 0.368 & 0.346 &       & 0.483 & 0.412 & 0.407 \\
    % 0.020 & 0.958 & 0.226 &       & 0.265 &       & 0.264 & 0.348 & 0.392 &       & 0.495 & 0.463 & 0.398 \\
    % 0.030 & 0.956 & 0.263 & 0.306 & 0.289 &       & 0.315 &       & 0.437 &       & 0.561 & 0.471 & 0.446 \\
    0.050 & 0.950 & 0.326 & 0.315 & 0.355 &       & 0.387 & 0.458 & 0.448 &       & 0.630 & 0.534 & 0.515 \\
    % 0.080 & 0.940 & 0.340 & 0.352 & 0.366 &       & 0.398 & 0.453 &       &       & 0.664 & 0.580 & 0.528 \\
    0.130 & 0.929 & 0.388 & 0.426 & 0.435 &       & 0.461 & \textbf{0.515} & 0.493 &       & 0.688 & 0.599 & 0.538 \\
    0.200 & 0.919 & 0.417 &       & \textbf{0.464} &       & 0.484 & 0.481 & 0.501 &       & 0.730 & 0.600 & 0.569 \\
    0.320 & 0.891 & \textbf{0.460} & 0.443 & 0.448 &       & 0.472 & 0.499 & \textbf{0.520} &       & \textbf{0.750} & \textbf{0.665} & 0.576 \\
    0.500 & 0.846 & 0.454 & \textbf{0.471} & \textbf{0.464} &       & \textbf{0.488} &       & 0.494 &       & 0.721 & 0.650 & \textbf{0.589} \\
          &       &       &       &       &       &       &       &       &       &       &       &  \\
    \textbf{Exponential (Sd)} &       &       &       &       &       &       &       &       &       &       &       &  \\
    \midrule
    0.010 & 0.953 & 0.153 & 0.174 &       &       & 0.228 & 0.292 & 0.306 &       & 0.443 & 0.404 & 0.395 \\
    % 0.015 & 0.958 & 0.187 &       & 0.224 &       & 0.228 & 0.349 & 0.358 &       & 0.448 & 0.455 & 0.389 \\
    % 0.020 & 0.959 & 0.201 & 0.241 & 0.253 &       & 0.277 & 0.336 & 0.347 &       & 0.506 & 0.480 & 0.433 \\
    0.050 & 0.953 & 0.312 & 0.326 & 0.330 &       & 0.343 & 0.468 & 0.435 &       & 0.616 & 0.575 & 0.479 \\
    % 0.080 & 0.945 & 0.340 & 0.375 &       &       & 0.395 & 0.470 & 0.471 &       & 0.661 & 0.559 & 0.510 \\
    0.130 & 0.940 & 0.373 & 0.402 & 0.411 &       & 0.424 &       & 0.504 &       & 0.679 & 0.585 & 0.526 \\
    0.200 & 0.936 & 0.394 &       & 0.414 &       & 0.455 & \textbf{0.510} & 0.501 &       & 0.701 & 0.623 & 0.550 \\
    0.320 & 0.919 & \textbf{0.429} & 0.426 & 0.416 &       & \textbf{0.494} & 0.492 & 0.513 &       & 0.739 & 0.638 & 0.564 \\
    0.500 & 0.900 & 0.423 & \textbf{0.454} & \textbf{0.470} &       & 0.488 & 0.494 & \textbf{0.516} &       & \textbf{0.752} & \textbf{0.699} & \textbf{0.594} \\
    \bottomrule
    \end{tabular}%
\end{table}%

\begin{table}[htbp]
  \centering
  \tiny
  \caption{Accuracy and Accuracy under attack of CIFAR-100 dataset.}
  \label{tab:cifar100-appendix}%
    \begin{tabular}{rcccccccccccc}
    \toprule
          & \multirow{2}[3]{*}{\textbf{Natural}} & \multicolumn{3}{c}{\textbf{$\ell_1$ -- EAD}} &       & \multicolumn{3}{c}{\textbf{$\ell_2$ -- C\&W}} &       & \multicolumn{3}{c}{\textbf{$\ell_\infty$ -- PGD}} \\
\cmidrule{3-5}\cmidrule{7-9}\cmidrule{11-13}          &       & \textbf{20} & \textbf{50} & \textbf{60} &       & \textbf{20} & \textbf{50} & \textbf{60} &       & \textbf{10} & \textbf{15} & \textbf{20} \\
    \textbf{Normal (Sd)} &       &       &       &       &       &       &       &       &       &       &       &  \\
    0.010 & 0.790 & 0.235 & 0.234 & 0.228 &       & 0.235 & 0.318 & 0.316 &       & 0.257 & 0.176 & 0.187 \\
    % 0.015 & 0.785 & 0.260 & 0.257 & 0.268 &       & 0.280 & 0.343 & 0.354 &       & 0.274 & 0.219 & 0.182 \\
    % 0.020 & 0.785 & 0.278 & 0.271 & 0.286 &       & 0.307 & 0.349 & 0.368 &       & 0.290 & 0.230 & 0.186 \\
    % 0.030 & 0.780 & 0.270 & 0.282 & 0.284 &       & 0.321 & 0.383 & 0.362 &       & 0.312 & 0.258 & 0.211 \\
    0.050 & 0.768 & 0.321 & 0.294 & 0.320 &       & 0.357 & 0.377 & 0.410 &       & 0.377 & 0.296 & 0.254 \\
    % 0.080 & 0.749 & 0.335 & 0.318 & 0.351 &       & \textbf{0.395} & \textbf{0.427} & 0.425 &       & 0.413 & 0.318 & 0.289 \\
    0.130 & 0.726 & \textbf{0.357} & \textbf{0.371} & 0.349 &       & 0.387 & \textbf{0.427} & \textbf{0.428} &       & 0.414 & 0.319 & 0.260 \\
    0.200 & 0.689 & 0.338 & 0.350 & \textbf{0.384} &       & 0.394 & 0.381 &       &       & 0.439 & 0.356 & 0.277 \\
    0.320 & 0.627 & 0.334 & 0.344 & 0.350 &       & 0.328 & 0.364 & 0.400 &       & \textbf{0.441} & 0.366 & 0.299 \\
    0.500 & 0.553 & 0.322 & 0.331 & 0.331 &       & 0.349 & 0.342 & 0.351 &       & 0.408 & \textbf{0.374} & \textbf{0.308} \\
          &       &       &       &       &       &       &       &       &       &       &       &  \\
    \textbf{Laplace (Sd)} &       &       &       &       &       &       &       &       &       &       &       &  \\
    \midrule
    0.010 & 0.782 & 0.199 & 0.227 & 0.243 &       & 0.225 & 0.311 & 0.321 &       & 0.236 & 0.190 & 0.177 \\
    % 0.015 & 0.781 & 0.246 & 0.260 & 0.276 &       & 0.243 & 0.349 & 0.342 &       & 0.250 & 0.216 & 0.196 \\
    % 0.020 & 0.784 & 0.262 & 0.267 & 0.284 &       & 0.292 & 0.391 & 0.335 &       & 0.290 & 0.218 & 0.173 \\
    % 0.030 & 0.773 &       & 0.288 & 0.294 &       & 0.319 & 0.396 & 0.393 &       & 0.342 & 0.268 & 0.209 \\
    0.050 & 0.763 & 0.326 & 0.317 & 0.331 &       & 0.354 & 0.377 & 0.409 &       & 0.368 & 0.319 & 0.256 \\
    % 0.080 & 0.751 & 0.343 & 0.315 & 0.358 &       & 0.380 & \textbf{0.422} & \textbf{0.431} &       & 0.394 & 0.329 & 0.260 \\
    0.130 & 0.723 & 0.337 & 0.357 & 0.344 &       & \textbf{0.408} & 0.414 & 0.408 &       & 0.420 & 0.346 & 0.293 \\
    0.200 & 0.695 & \textbf{0.355} & 0.349 & \textbf{0.361} &       & 0.393 & 0.405 & 0.393 &       & 0.445 & 0.340 & 0.303 \\
    0.320 & 0.647 & 0.324 & \textbf{0.373} & 0.357 &       & 0.388 & 0.387 & 0.373 &       & \textbf{0.460} & 0.381 & 0.303 \\
    0.500 & 0.572 & 0.310 & 0.308 & 0.323 &       & 0.358 & 0.351 & 0.361 &       & 0.425 & \textbf{0.403} & \textbf{0.329} \\
          &       &       &       &       &       &       &       &       &       &       &       &  \\
    \textbf{Exponential (Sd)} &       &       &       &       &       &       &       &       &       &       &       &  \\
    \midrule
    0.010 & 0.785 & 0.218 & 0.251 & 0.217 &       & 0.247 & 0.278 & 0.321 &       & 0.250 & 0.214 & 0.169 \\
    % 0.015 & 0.783 & 0.247 & 0.279 &       &       & 0.243 &       & 0.343 &       & 0.266 & 0.197 & 0.168 \\
    % 0.020 & 0.785 & 0.261 & 0.296 & 0.283 &       & 0.276 & 0.373 & 0.364 &       & 0.306 & 0.224 & 0.187 \\
    % 0.030 & 0.782 & 0.292 & 0.310 & 0.307 &       & 0.374 & 0.381 & 0.362 &       & 0.312 & 0.279 & 0.219 \\
    0.050 & 0.767 & 0.323 & 0.337 & 0.317 &       & 0.346 & 0.380 & 0.402 &       & 0.356 & 0.291 & 0.235 \\
    % 0.080 & 0.755 & 0.337 & 0.352 & 0.344 &       & 0.397 & 0.380 &       &       & 0.390 & 0.313 & 0.261 \\
    0.130 & 0.749 & 0.330 &       & 0.356 &       & \textbf{0.403} & \textbf{0.444} & \textbf{0.421} &       & 0.400 & 0.328 & 0.266 \\
    0.200 & 0.731 & 0.345 & 0.361 & 0.357 &       & 0.388 & 0.424 & 0.406 &       & 0.427 & 0.340 & 0.267 \\
    0.320 & 0.703 & 0.349 & 0.351 & 0.340 &       & 0.388 & 0.439 & 0.399 &       & 0.433 & 0.351 & 0.280 \\
    0.500 & 0.673 & \textbf{0.387} & \textbf{0.378} & \textbf{0.378} &       & 0.396 & 0.435 &       &       & \textbf{0.485} & \textbf{0.370} & \textbf{0.322} \\
    \bottomrule
    \end{tabular}%
\end{table}%

\subsection{Large scale robustness}
Adversarial training fails to generalize to higher dimensional datasets such as ImageNet. We conducted experiments with the large scale ImageNet dataset and compared our randomized neural network against large scale adversarial training proposed by Kurakin et al.~\cite{kurakin2016adversarial}. One can observe from Table~\ref{tab:adv_imagenet-appendix} that the model from Kurakin et al. is neither robust against recent $\ell_1$ nor $\ell_2$  iterative attacks such as EAD and C\&W. Moreover, it offers a small robustness against $\ell_\infty$ PGD attack. Our randomized neural network with $\EoT$ attacks offers a small robustness on $\ell_1$ and $\ell_2$ attacks while being less robust against PGD. 
\begin{table}[htb]
  \centering
%   \scriptsize
  \caption{Accuracy under attack of the Adversarial model training by Kurakin et al. \cite{kurakin2016adversarial} and an Inception Resnet v2 model training with normal 0.1 noise injected in the image on the ImageNet dataset. }
  \label{tab:adv_imagenet-appendix}
  \centering
  \begin{tabular}{lcccccc}
    \toprule
      & \textbf{Baseline} & \textbf{$\ell_1$ EAD 60} & \textbf{$\ell_2$ C\&W 60} & \textbf{$\ell_\infty$ PGD} \\
    \midrule
    \textbf{Kurakin et al. \cite{kurakin2016adversarial}} & 0.739 & 0.097 & 0.100 & 0.239 \\
    \textbf{Normal 0.1} & 0.625 & 0.255 & 0.301 & 0.061 \\
    \bottomrule
  \end{tabular}
\end{table}

\subsection{Experiments with noise on the first activation}
The aim of the following experiments is empirically illustrate the \textit{Data processing inequality} in Proposition~\ref{prop::postprocessing-appendix}. 

%\todo{FY:Cette experience est faite pour valider empiriquement le fait que l'injection de bruit peut être faite n'importe ou dans le reseau. Si c'est bien cela, il faudrait le dire et faire reference au resultat theorique a ce propos. }
%\todo{LM: on aurait du le faire mais too late... }
Table~\ref{tab:acc_noise_activation-appendix} and \ref{tab:attack_noise_activation-appendix} present the experiments conducted with the same set of parameters as the previous ones  on CIFAR-10 and CIFAR-100, but with the noise injected in the first activation layer instead of directly in the image. We observe from  Table~\ref{tab:acc_noise_activation-appendix} that we can inject more noise with a marginal loss on accuracy. The accuracy under attack is presented in Table \ref{tab:attack_noise_activation-appendix} for a selection of models. 

\begin{table}[htbp]
  \centering
%   \scriptsize
  \caption{Impact of the distribution and the intensity of the noise with randomized networks with noise injected on the first activation}
  \label{tab:acc_noise_activation-appendix}%
    \begin{tabular}{rcrrcrrc}
    \toprule
    \multicolumn{1}{c}{\textbf{Sd}} & \textbf{Normal} &       & \multicolumn{1}{c}{\textbf{Sd}} & \textbf{Laplace} &       & \multicolumn{1}{c}{\textbf{Sd}} & \textbf{Exponential} \\
\cmidrule{1-2}\cmidrule{4-5}\cmidrule{7-8}    0.01  & 0.956 &       & 0.01  & 0.955 &       & 0.01  & 0.953 \\
    0.23  & 0.943 &       & 0.05  & 0.947 &       & 0.08  & 0.943 \\
    0.45  & 0.935 &       & 0.10  & 0.933 &       & 0.15  & 0.938 \\
    0.68  & 0.926 &       & 0.15  & 0.916 &       & 0.23  & 0.925 \\
    0.90  & 0.916 &       & 0.20  & 0.911 &       & 0.30  & 0.919 \\
    1.00  & 0.916 &       & 0.25  & 0.897 &       & 0.38  & 0.903 \\
    1.34  & 0.906 &       & 0.30  & 0.889 &       & 0.45  & 0.897 \\
    1.55  & 0.900 &       & 0.35  & 0.882 &       & 0.53  & 0.886 \\
    1.77  & 0.893 &       & 0.40  & 0.867 &       & 0.60  & 0.885 \\
    2.00  & 0.886 &       & 0.45  & 0.855 &       & 0.68  & 0.875 \\
    \bottomrule
    \end{tabular}%
\end{table}%

\begin{table}[htbp]
  \begin{adjustwidth}{-1.5cm}{}
%   \scriptsize
  \caption{Accuracy and Accuracy under attack of selected models with noise on the first activation}
   \label{tab:attack_noise_activation-appendix}%
    \begin{tabular}{llcccccccccccc}
    \toprule
          &       &       & \multirow{2}[3]{*}{\textbf{Natural}} & \multicolumn{3}{c}{\textbf{$\ell_1$ -- EAD}} &       & \multicolumn{3}{c}{\textbf{$\ell_2$ -- C\&W}} &       & \multicolumn{2}{c}{\textbf{$\ell_\infty$ -- PGD}} \\
\cmidrule{5-7}\cmidrule{9-11}\cmidrule{13-14}          &       &       &       & \textbf{20} & \textbf{50} & \textbf{60} &       & \textbf{20} & \textbf{50} & \textbf{60} &       & \textbf{10} & \textbf{20} \\
    \textbf{Dataset} & \multicolumn{1}{c}{\textbf{Distribution}} & \textbf{Sd} &       &       &       &       &       &       &       &       &       &       &  \\
    \midrule
    \multirow{3}[1]{*}{CIFAR10} & Normal & 1.55  & 0.900 & 0.441 & 0.440 & 0.413 &       & 0.477 & 0.482 & 0.484 &       & 0.683 & 0.526 \\
          & Laplace & 0.25  & 0.897 & 0.388 & 0.436 & \textbf{0.445} &       & 0.481 & 0.506 & 0.491 &       & 0.664 & 0.493 \\
          & Exponential & 0.38  & \textbf{0.903} & \textbf{0.456} & \textbf{0.463} & 0.438 &       & \textbf{0.495} & \textbf{0.516} & \textbf{0.506} &       & \textbf{0.697} & \textbf{0.557} \\
    \midrule
    \multirow{3}[2]{*}{CIFAR100} & Normal & 0.45  & 0.741 & \textbf{0.362} & 0.352 & 0.353 &       & 0.352 & 0.410 & 0.418 &       & 0.380 & 0.250 \\
          & Laplace & 0.10  & \textbf{0.742} & 0.350 & \textbf{0.367} & 0.350 &       & 0.371 & \textbf{0.419} &       &       & 0.418 & \textbf{0.264} \\
          & Exponential & 0.15  & 0.741 & 0.354 & 0.356 & \textbf{0.373} &       & \textbf{0.394} & 0.409 & \textbf{0.420} &       & \textbf{0.430} & 0.258 \\
    \bottomrule
    \end{tabular}%
    \end{adjustwidth}
\end{table}%

\FloatBarrier

\section{Additional discussions on the experiments}
For the sake of completeness and reproducibility, we give some additional insights on the noise injection scheme and comprehensive details on our numerical experiments.
\subsection{On the need for injecting noise in the training phase}
\label{section::covariateshift-appendix}
Robustness has always been thought as a property to be enforced at inference time and it is tempting to focus only on injecting noise at inference. However, simply doing so ruins the accuracy of the algorithm (as it becomes an instance of distribution shift~\cite{sugiyama2012machine}). Indeed, making the assumption that the training and test distributions matches, in practice, injecting some noise at inference would result in changing the test distribution.%\todo{FY (a discuter) : on pourrait ajouter que dans le meilleur des cas, avec des distributions gaussianes en apprentissage et en test, injecter du bruit gaussian à l'inference aura pour effet d'augmenter la variance de la distribution de test.}

Distribution shift occurs when the training distribution differs from the test distribution. This implies that the hypothesis minimizing the empirical risk is not consistent, i.e. it does not converge to the true model as the training size increases. A way to circumvent that is to ensure that training and test distributions matches using importance weighting (in the case of covariate-shift) or with noise injection in the training phases as well (in our case).

\subsection{Reproducibility of the  experiments} We emphasize that all experiments should be easily reproducible. All our experiments are developed with TensorFlow version 1.12~\cite{tensorflow2015-whitepaper}. The code is available as supplemental material and will be open sourced upon acceptance of the paper. The archive contains a {\em readme} file containing a small documentation on how to run the experiments, a configuration file which defines the architecture and the hyper-parameters of the experiments, python scripts which generate a bash command to run the experiments. The code contains Randomized Wide Resnet used for CIFAR-10 and CIFAR100, Inception Resnet v2 used for ImageNet, PGD, EAD and C\&W attacks used for evaluation under attack. We ran our experiments, on a cluster with computers each having 8 GPU Nvidia V100. 

\end{adjustwidth}

%% file: main.bbl
\begin{thebibliography}{10}

\bibitem{tensorflow2015-whitepaper}
M.~Abadi, A.~Agarwal, P.~Barham, E.~Brevdo, Z.~Chen, C.~Citro, G.~S. Corrado,
  A.~Davis, J.~Dean, M.~Devin, S.~Ghemawat, I.~Goodfellow, A.~Harp, G.~Irving,
  M.~Isard, Y.~Jia, R.~Jozefowicz, L.~Kaiser, M.~Kudlur, J.~Levenberg,
  D.~Man\'{e}, R.~Monga, S.~Moore, D.~Murray, C.~Olah, M.~Schuster, J.~Shlens,
  B.~Steiner, I.~Sutskever, K.~Talwar, P.~Tucker, V.~Vanhoucke, V.~Vasudevan,
  F.~Vi\'{e}gas, O.~Vinyals, P.~Warden, M.~Wattenberg, M.~Wicke, Y.~Yu, and
  X.~Zheng.
\newblock {TensorFlow}: Large-scale machine learning on heterogeneous systems,
  2015.
\newblock Software available from tensorflow.org.

\bibitem{athalye2018obfuscated}
A.~Athalye, N.~Carlini, and D.~Wagner.
\newblock Obfuscated gradients give a false sense of security: Circumventing
  defenses to adversarial examples.
\newblock In J.~Dy and A.~Krause, editors, {\em Proceedings of the 35th
  International Conference on Machine Learning}, volume~80 of {\em Proceedings
  of Machine Learning Research}, pages 274--283, Stockholmsmässan, Stockholm
  Sweden, 10--15 Jul 2018. PMLR.

\bibitem{athalye2017synthesizing}
A.~Athalye, L.~Engstrom, A.~Ilyas, and K.~Kwok.
\newblock Synthesizing robust adversarial examples.
\newblock {\em arXiv preprint arXiv:1707.07397}, 2017.

\bibitem{beaudry2011intuitive}
N.~J. Beaudry and R.~Renner.
\newblock An intuitive proof of the data processing inequality.
\newblock {\em Quantum Info. Comput.}, 12(5-6):432--441, May 2012.

\bibitem{ben2009robust}
A.~Ben-Tal, L.~El~Ghaoui, and A.~Nemirovski.
\newblock {\em Robust optimization}, volume~28.
\newblock Princeton University Press, 2009.

\bibitem{carlini2019evaluating}
N.~Carlini, A.~Athalye, N.~Papernot, W.~Brendel, J.~Rauber, D.~Tsipras,
  I.~Goodfellow, and A.~Madry.
\newblock On evaluating adversarial robustness.
\newblock {\em arXiv preprint arXiv:1902.06705}, 2019.

\bibitem{carlini2017towards}
N.~Carlini and D.~Wagner.
\newblock Towards evaluating the robustness of neural networks.
\newblock In {\em 2017 IEEE Symposium on Security and Privacy (SP)}, pages
  39--57. IEEE, 2017.

\bibitem{ChapR04}
F.~Chapeau-Blondeau and D.~Rousseau.
\newblock Noise-enhanced performance for an optimal bayesian estimator.
\newblock {\em IEEE Transactions on Signal Processing}, 52(5):1327--1334, 2004.

\bibitem{chen2018ead}
P.-Y. Chen, Y.~Sharma, H.~Zhang, J.~Yi, and C.-J. Hsieh.
\newblock Ead: elastic-net attacks to deep neural networks via adversarial
  examples.
\newblock In {\em Thirty-second AAAI conference on artificial intelligence},
  2018.

\bibitem{KolterRandomizedSmoothing}
J.~M. Cohen, E.~Rosenfeld, and J.~Z. Kolter.
\newblock Certified adversarial robustness via randomized smoothing.
\newblock {\em CoRR}, abs/1902.02918, 2019.

\bibitem{cover2012elements}
T.~M. Cover and J.~A. Thomas.
\newblock {\em Elements of information theory}.
\newblock John Wiley \& Sons, 2012.

\bibitem{imagenet_cvpr09}
J.~Deng, W.~Dong, R.~Socher, L.-J. Li, K.~Li, and L.~Fei-Fei.
\newblock {ImageNet: A Large-Scale Hierarchical Image Database}.
\newblock In {\em CVPR09}, 2009.

\bibitem{pruningDefenseICLR2018}
G.~S. Dhillon, K.~Azizzadenesheli, J.~D. Bernstein, J.~Kossaifi, A.~Khanna,
  Z.~C. Lipton, and A.~Anandkumar.
\newblock Stochastic activation pruning for robust adversarial defense.
\newblock In {\em International Conference on Learning Representations}, 2018.

\bibitem{NIPS2018Mahloujifar}
D.~Diochnos, S.~Mahloujifar, and M.~Mahmoody.
\newblock Adversarial risk and robustness: General definitions and implications
  for the uniform distribution.
\newblock In {\em Advances in Neural Information Processing Systems}, pages
  10380--10389, 2018.

\bibitem{dwork2014algorithmic}
C.~Dwork, A.~Roth, et~al.
\newblock The algorithmic foundations of differential privacy.
\newblock {\em Foundations and Trends{\textregistered} in Theoretical Computer
  Science}, 9(3--4):211--407, 2014.

\bibitem{NIPS2018Fawzi}
A.~Fawzi, H.~Fawzi, and O.~Fawzi.
\newblock Adversarial vulnerability for any classifier.
\newblock In S.~Bengio, H.~Wallach, H.~Larochelle, K.~Grauman, N.~Cesa-Bianchi,
  and R.~Garnett, editors, {\em Advances in Neural Information Processing
  Systems 31}, pages 1186--1195. Curran Associates, Inc., 2018.

\bibitem{Moosavi2016Robustnessofaclassifier}
A.~Fawzi, S.-M. Moosavi-Dezfooli, and P.~Frossard.
\newblock Robustness of classifiers: from adversarial to random noise.
\newblock In {\em Advances in Neural Information Processing Systems}, pages
  1632--1640, 2016.

\bibitem{fawzi2018empirical}
A.~Fawzi, S.-M. Moosavi-Dezfooli, P.~Frossard, and S.~Soatto.
\newblock Empirical study of the topology and geometry of deep networks.
\newblock In {\em IEEE CVPR}, 2018.

\bibitem{AGibbsMetrics2002}
A.~L. Gibbs and F.~E. Su.
\newblock On choosing and bounding probability metrics.
\newblock {\em International Statistical Review / Revue Internationale de
  Statistique}, 70(3):419--435, 2002.

\bibitem{5605338}
G.~L. {Gilardoni}.
\newblock On pinsker's and vajda's type inequalities for
  csiszár's$f$-divergences.
\newblock {\em IEEE Transactions on Information Theory}, 56(11):5377--5386, Nov
  2010.

\bibitem{goodfellow2014explaining}
I.~Goodfellow, J.~Shlens, and C.~Szegedy.
\newblock Explaining and harnessing adversarial examples.
\newblock In {\em International Conference on Learning Representations}, 2015.

\bibitem{gouk2018regularisation}
H.~Gouk, E.~Frank, B.~Pfahringer, and M.~Cree.
\newblock Regularisation of neural networks by enforcing lipschitz continuity.
\newblock {\em arXiv preprint arXiv:1804.04368}, 2018.

\bibitem{guo2017countering}
C.~Guo, M.~Rana, M.~Cisse, and L.~van~der Maaten.
\newblock Countering adversarial images using input transformations.
\newblock In {\em International Conference on Learning Representations}, 2018.

\bibitem{he2016deep}
K.~He, X.~Zhang, S.~Ren, and J.~Sun.
\newblock Deep residual learning for image recognition.
\newblock In {\em Proceedings of the IEEE conference on computer vision and
  pattern recognition}, pages 770--778, 2016.

\bibitem{huber2011robust}
P.~J. Huber.
\newblock Robust statistics.
\newblock In {\em International Encyclopedia of Statistical Science}, pages
  1248--1251. Springer, 2011.

\bibitem{krizhevsky2009learning}
A.~Krizhevsky and G.~Hinton.
\newblock Learning multiple layers of features from tiny images.
\newblock Technical report, Citeseer, 2009.

\bibitem{kurakin2016adversarial}
A.~Kurakin, I.~Goodfellow, and S.~Bengio.
\newblock Adversarial examples in the physical world.
\newblock {\em arXiv preprint arXiv:1607.02533}, 2016.

\bibitem{lecuyer2018certified}
M.~Lecuyer, V.~Atlidakis, R.~Geambasu, D.~Hsu, and S.~Jana.
\newblock Certified robustness to adversarial examples with differential
  privacy.
\newblock In {\em 2019 IEEE Symposium on Security and Privacy (SP)}, pages
  727--743, 2018.

\bibitem{SecondOrdercertifiedrobustness}
B.~Li, C.~Chen, W.~Wang, and L.~Carin.
\newblock Second-order adversarial attack and certifiable robustness.
\newblock {\em CoRR}, abs/1809.03113, 2018.

\bibitem{Xuang2018}
X.~Liu, M.~Cheng, H.~Zhang, and C.-J. Hsieh.
\newblock Towards robust neural networks via random self-ensemble.
\newblock In {\em European Conference on Computer Vision}, pages 381--397.
  Springer, 2018.

\bibitem{madry2017towards}
A.~Madry, A.~Makelov, L.~Schmidt, D.~Tsipras, and A.~Vladu.
\newblock Towards deep learning models resistant to adversarial attacks.
\newblock In {\em International Conference on Learning Representations}, 2018.

\bibitem{madry2018towards}
A.~Madry, A.~Makelov, L.~Schmidt, D.~Tsipras, and A.~Vladu.
\newblock Towards deep learning models resistant to adversarial attacks.
\newblock In {\em International Conference on Learning Representations}, 2018.

\bibitem{meng2017magnet}
D.~Meng and H.~Chen.
\newblock Magnet: a two-pronged defense against adversarial examples.
\newblock In {\em Proceedings of the 2017 ACM SIGSAC Conference on Computer and
  Communications Security}, pages 135--147. ACM, 2017.

\bibitem{MitaK98}
S.~Mitaim and B.~Kosko.
\newblock Adaptive stochastic resonance.
\newblock {\em Proceedings of the IEEE}, 86(11):2152--2183, 1998.

\bibitem{moosavi2017universal}
S.-M. Moosavi-Dezfooli, A.~Fawzi, O.~Fawzi, and P.~Frossard.
\newblock Universal adversarial perturbations.
\newblock In {\em 2017 IEEE Conference on Computer Vision and Pattern
  Recognition (CVPR)}, pages 86--94. Ieee, 2017.

\bibitem{moosavi2016deepfool}
S.-M. Moosavi-Dezfooli, A.~Fawzi, and P.~Frossard.
\newblock Deepfool: a simple and accurate method to fool deep neural networks.
\newblock In {\em Proceedings of the IEEE Conference on Computer Vision and
  Pattern Recognition}, pages 2574--2582, 2016.

\bibitem{Papernot2016TheLO}
N.~Papernot, P.~McDaniel, S.~Jha, M.~Fredrikson, Z.~B. Celik, and A.~Swami.
\newblock The limitations of deep learning in adversarial settings.
\newblock In {\em Security and Privacy (EuroS\&P), 2016 IEEE European Symposium
  on}, pages 372--387. IEEE, 2016.

\bibitem{papernot2016distillation}
N.~Papernot, P.~McDaniel, X.~Wu, S.~Jha, and A.~Swami.
\newblock Distillation as a defense to adversarial perturbations against deep
  neural networks.
\newblock In {\em 2016 IEEE Symposium on Security and Privacy (SP)}, pages
  582--597. IEEE, 2016.

\bibitem{Perez2017TheEO}
L.~Perez and J.~Wang.
\newblock The effectiveness of data augmentation in image classification using
  deep learning.
\newblock {\em arXiv preprint arXiv:1712.04621}, 2017.

\bibitem{rakin2018parametricnoiseinjection}
A.~S. Rakin, Z.~He, and D.~Fan.
\newblock Parametric noise injection: Trainable randomness to improve deep
  neural network robustness against adversarial attack.
\newblock {\em arXiv preprint arXiv:1811.09310}, 2018.

\bibitem{renyi1961}
A.~R{\'e}nyi.
\newblock On measures of entropy and information.
\newblock Technical report, HUNGARIAN ACADEMY OF SCIENCES Budapest Hungary,
  1961.

\bibitem{Samangouei2018DefenseGAN}
P.~Samangouei, M.~Kabkab, and R.~Chellappa.
\newblock Defense-{GAN}: Protecting classifiers against adversarial attacks
  using generative models.
\newblock In {\em International Conference on Learning Representations}, 2018.

\bibitem{simon2018adversarial}
C.-J. Simon-Gabriel, Y.~Ollivier, B.~Sch{\"o}lkopf, L.~Bottou, and
  D.~Lopez-Paz.
\newblock Adversarial vulnerability of neural networks increases with input
  dimension.
\newblock {\em arXiv preprint arXiv:1802.01421}, 2018.

\bibitem{sugiyama2012machine}
M.~Sugiyama and M.~Kawanabe.
\newblock {\em Machine learning in non-stationary environments: Introduction to
  covariate shift adaptation}.
\newblock MIT press, 2012.

\bibitem{szegedy2017inception}
C.~Szegedy, S.~Ioffe, V.~Vanhoucke, and A.~A. Alemi.
\newblock Inception-v4, inception-resnet and the impact of residual connections
  on learning.
\newblock In {\em Thirty-First AAAI Conference on Artificial Intelligence},
  2017.

\bibitem{Szegedy2013IntriguingPO}
C.~Szegedy, W.~Zaremba, I.~Sutskever, J.~Bruna, D.~Erhan, I.~Goodfellow, and
  R.~Fergus.
\newblock Intriguing properties of neural networks.
\newblock In {\em International Conference on Learning Representations}, 2014.

\bibitem{Vajda1970}
I.~Vajda.
\newblock Note on discrimination information and variation.
\newblock {\em IEEE Trans. Inform. Theory}, 16(6):771--773, Nov. 1970.

\bibitem{villani2008optimal}
C.~Villani.
\newblock {\em Optimal transport: old and new}, volume 338.
\newblock Springer Science \& Business Media, 2008.

\bibitem{Xie2017MitigatingAE}
C.~Xie, J.~Wang, Z.~Zhang, Z.~Ren, and A.~Yuille.
\newblock Mitigating adversarial effects through randomization.
\newblock In {\em International Conference on Learning Representations}, 2018.

\bibitem{ZagoruykoK16}
S.~Zagoruyko and N.~Komodakis.
\newblock Wide residual networks.
\newblock In E.~R.~H. Richard C.~Wilson and W.~A.~P. Smith, editors, {\em
  Proceedings of the British Machine Vision Conference (BMVC)}, pages
  87.1--87.12. BMVA Press, September 2016.

\bibitem{ZozoA99}
S.~Zozor and P.-O. Amblard.
\newblock Stochastic resonance in discrete time nonlinear {AR}(1) models.
\newblock {\em IEEE transactions on Signal Processing}, 47(1):108--122, 1999.

\end{thebibliography}
